\documentclass[10pt,journal,compsoc]{IEEEtran}
\usepackage{amsfonts}
\usepackage{amsmath}
\usepackage{amssymb}
\usepackage{amsthm}
\usepackage{mathrsfs}
\usepackage{float}
\usepackage{graphicx}
\usepackage{tabularx}
\usepackage{algorithm} 
\usepackage{algorithmic}
\usepackage{xcolor}
\usepackage{caption}
\usepackage{subfigure}
\usepackage{epstopdf}

\usepackage{cite}
\usepackage{hyperref}
\usepackage{tikz}
\usetikzlibrary{shapes,arrows}

\numberwithin{equation}{section}
\newtheorem{proposition}{P{\scriptsize ROPOSITION}}[section]
\newtheorem{theorem}[proposition]{T{\scriptsize HEOREM}}
\newtheorem{lemma}[proposition]{L{\scriptsize  EMMA}}

\newtheorem{remark}{R{\scriptsize  EMARK}}[section]
\newtheorem{definition}{D{\scriptsize  EFINITION}}[section]

\newtheorem{example}{E{\scriptsize  xample}}[section]

\def\R{\mathbb{R}}

\def\H{\mathbb{Q}}
\def\Q{\mathbb{Q}}

\newcommand{\iq}{{\bf i}}
\newcommand{\Iq}{{\bf I}}
\newcommand{\jq}{{\bf j}}
\newcommand{\kq}{{\bf k}}

\newcommand{\Aq}{{\bf A}}
\newcommand{\aq}{{\bf a}}

\newcommand{\vq}{{\bf v}}

\newcommand{\uq}{{\bf u}}

\newcommand{\Pq}{{\bf P}}

\newcommand{\Wq}{{\bf W}}
\newcommand{\wq}{{\bf w}}

\newcommand{\Fq}{{\bf F}}

\newcommand{\Dq}{{\bf D}}

\begin{document}
\title {Generalized Two-Dimensional Quaternion Principal Component Analysis with Weighting\\ for Color Image Recognition\thanks{The work of the authors is partially supported  by  National Natural Science Foundation of China under grants 12171210, 12090011 and 11771188; 
the Natural Science Foundation of the Jiangsu
Higher Education Institutions of China under grant 21KJA110001; the Priority Academic Program Development Project (PAPD);  the Top-notch Academic Programs Project (No. PPZY2015A013) of Jiangsu Higher Education Institutions.}}

\author{Zhi-Gang~Jia,~
        Zi-Jin~Qiu,~
        Qian-Yu~Wang,~
        Mei-Xiang~Zhao
        and~Dan-Dan~Zhu
\IEEEcompsocitemizethanks{\IEEEcompsocthanksitem Z.-G. Jia is with the School of Mathematics and Statistics,  Jiangsu Normal University, Xuzhou 221116,
China; and with  the Research Institute of Mathematical Science,  Jiangsu Normal University, Xuzhou 221116,
China.\protect\\
E-mail: zhgjia@jsnu.edu.cn
\IEEEcompsocthanksitem Z.-J. Qiu, Q.-Y. Wang and D.-D. Zhu are with the School of Mathematics and Statistics,  Jiangsu Normal University, Xuzhou 221116,
China.
\IEEEcompsocthanksitem 
M.-X. Zhao is with School of Information and Control Engineering,
China University of Mining and Technology,
Xuzhou 221116, China; and with  the School of Mathematics and Statistics,  Jiangsu Normal University, Xuzhou 221116,
China.

}
}

\IEEEtitleabstractindextext{%
\begin{abstract}
 One of the most powerful methods of color image recognition is the two-dimensional principle component analysis (2DQPCA) approach, which is based on  quaternion representation and preserves color information very well.  However, the current versions of 2DQPCA are still not feasible to extract different geometric properties of color images according to practical data analysis requirements and  they are vulnerable to strong noise. In this paper, a generalized 2DQPCA approach with weighting is presented with imposing $L_{p}$ norms on both constraint  and objective functions. As  a unit  2DQPCA framework, this new version makes it possible to choose adaptive regularizations and constraints according to  actual  applications and  can extract both geometric properties and color information of color images.  The projection vectors generated by the deflating scheme are required to be orthogonal to each other.  A weighting matrix is defined to magnify the effect of main features. This overcomes the shortcomings of traditional 2DQPCA that the recognition rate decreases as the number of principal components increases. The numerical results based on the real face databases validate that the newly proposed method is robust to noise and  performs better than the state-of-the-art 2DQPCA-based algorithms and four  prominent deep learning methods.
\end{abstract}

}
\maketitle

\IEEEraisesectionheading{\section{Introduction}\label{sec:introduction}}
\noindent
\IEEEPARstart{C}olor image recognition becomes more and more important in recent data science.
Two dimensional quaternion principle analysis  was proposed to scratch the features of color face images  \cite{jlz17:2DQPCA} and  had been well developed for the dimensional reduction  and the color image reconstruction  \cite{zjg19:im-2dqpca,xz19:2DPCA-S,xcgz20:2DQS,zjccg20:Advanced2DPCA}.  This contributes to the development of the fundamental tool---the principal component analysis---of  machine learning frameworks.  The two dimensional quaternion principle analysis based approaches have the advantages at  preserving the spatial structure and color information of color images and costing less computational operations. However, the geometric properties (such as sparsity) of color images are not well extracted and the current versions of two dimensional quaternion principle analysis are still not robust to treat color images polluted by noise.   
In this paper, we present a generalized two dimensional quaternion principle analysis with utilizing $L_p$-norms of quaternion vectors to obtain more  geometrical and color information and to resist more kinds of noises.  Our aim is to present a new approach for color image recognition, which  significantly promotes the robustness and the  recognition rate.

It is well known that  two dimensional quaternion principle analysis (2DQPCA) generalizes two dimensional  principle analysis (2DPCA)   \cite{yzfy04:2dpca} to the quaternion skew-field with a strong motivation of applying the color information of color images.  Recall that principal component analysis (PCA) \cite{tp91} is an unsupervised learning approach for feature extraction and dimension reduction. It has been widely used in the fields of computer vision and pattern recognition. Recently, many robust PCA (RPCA)   algorithms are proposed with improving the quadratic formulation, which renders PCA vulnerable to noises, into $L_{1}$-norm on the objection function, e.g., $L_{1}$-PCA \cite{kk05:L1pca}, $R_{1}$-PCA \cite{dz06:R1pca}, and PCA-$L_{1}$ \cite{KN08:pcaL1}. Besides of robustness, sparsity is also a desired property. By applying $L_{0}$-norm or $L_{1}$-norm on the constraint function (also called penalty or regularization \cite{sl09,zh05}) of PCA, a series of sparse PCA (SPCA) algorithms  have been proposed in \cite{zht06:spca,sh08:spca,jnrs10:spca,ojht16:spca,rsbx18:spca}. In order to enhance both robustness and sparsity at the same time, robust sparse PCA (RSPCA) is further developed, see \cite{mzx12:rspca} and \cite{pzlw19:rspca} for instance.  Considering that $L_{0}$-norm, $L_{1}$-norm and $L_{2}$-norm are all special cases of $L_{p}$-norm, it is natural to replace the $L_{2}$-norm in traditional PCA with arbitrary norm, both on its objective function and constraint function, as proposed in generalized PCA (GPCA)\cite{lxz13:gpca}, \cite{nk14:gpca}.

When being applied to extract feature from images, PCA treats each sample as a vector, and hence, 2D images are converted to high-dimensional vectors prior to feature extraction \cite{tp91}. To avoid the intensive computation of high-dimensional data, Yang \textit{et al.} \cite{yzfy04:2dpca} proposed the 2DPCA approach, which constructs the covariance matrix by directly using 2D face image matrices. 2DPCA directly constructs the sample covariance matrix from the two-dimensional image, which reduces the burden of calculating the sample covariance matrix from the high-dimensional vectors, and therefore has higher computational efficiency. In addition, 2DPCA processes a two-dimensional matrix to preserve the spatial structure of the images \cite{yzfy04:2dpca,yj05:2dpca,zz05:2dpca}, and achieves a higher face recognition rate than PCA in most cases.  This image-as-matrix method offers insights for improving above RSPCA, PCA-$L_{p}$, GPCA, etc.  The $L_{1}$-norm based 2DPCA (2DPCA-$L_{1}$)\cite{lpy10:2dpcaL1} and 2DPCA-$L_{1}$ with sparsity (2DPCA-$L_{1}$S) \cite{mz12:RS2DPCA} are two typical improvements of PCA-$L_{1}$ and RSPCA, respectively. And the generalized 2DPCA (G2DPCA) \cite{jw16:g2dpca} imposes $L_{p}$-norm on both objective and constraint functions of 2DPCA. Recently, Chen \textit{et al.} \cite{cjcz19:R2DPCA} further proposed the R2DPCA algorithm, which utilizes the label information (if known) of training samples to calculate a relaxation vector and presents a weight to each subset of training data.

To process color images, the above method separately processes three color channels or connects the representations of different color channels into a large matrix, so cross-channel correlation is not considered \cite{mes08}, \cite{lz16}. However, this correlation is very important for color image processing. Xiang \textit{et al.} \cite{xyc15:CPCA} proposed a CPCA approach for color face recognition. They utilized a color image matrix representation model based on the framework of PCA and
applied 2DPCA to compute the optimal projection for feature extraction. Zou \textit{et al.} \cite{zkw16:QColl} then presented quaternion collaborative representation-based classification (QCRC) and quaternion sparse RC (QSRC)  using quaternion $L_1$ minimization. For color face recognition, a series of quaternion-based methods, such
as the quaternion PCA (QPCA) \cite{bisa03:QPCA}, the two-dimensional QPCA (2DQPCA), the bidirectional 2DQPCA \cite{scy11:bid2DQPCA}, the kernel QPCA (KQPCA) and the two-dimensional KQPCA \cite{{cyjz17:KQPCA}}, have been proposed, with generalizing the conventional PCA and 2DPCA. Recently, Jia \textit{et al.}\cite{jlz17:2DQPCA} presented the 2DQPCA approach based on quaternion models with reducing the feature dimension in row direction.
 Xiao and Zhou \cite{xz19:2DPCA-S} proposed novel quaternion ridge regression models for 2DQPCA with reducing the feature dimension in column direction and two-dimensional quaternion sparse principle component analysis. Zhao \textit{et al.} \cite{zjg19:im-2dqpca} further improved the 2DQPCA into abstracting the features of quaternion matrix samples in both row and column directions.
 Lately, Xiao \textit{et al.} \cite{xcgz20:2DQS} proposed a two-dimensional quaternion sparse discriminant analysis (2D-QSDA), including sparse regularization, that meets the requirements of representing RGB and RGB-D images.
  These 2DQPCA-like approaches can preserve the spatial structure of color images and have a low computation cost. They have achieved a significant success in promoting the robustness and the ratio of face recognition by utilizing color information.

In addition to PCA-like methods, LDA and its variants are still one kind of feature extraction algorithms which play an important role in pattern recognition and computer vision; see \cite{yl05:2DLDA,nks06:2DLDA,lls08:2DLDA} for instance.  Besides of color face recognition, many recent studies (e.g. \cite{jns19:lansvd,jns19:QMC,jnw19,cxz20:LRQ}) also have shown that the quaternion framework is well adapted to color image restoration by encoding the color channels into the three imaginary parts.

Nowadays, the color face recognition problem faces more complicated backgrounds, noise interference, and higher recognition accuracy requirements. There is still a lack of  mathematical models that can reflect physical properties, are more interpretable, and are better than neural network methods. Thus, we took a small step towards this goal, established a general mathematical model and established a comprehensive and systematic solution method. 
The contribution of this paper is listed  in three aspects.
\begin{itemize}
\item A novel generalized two-dimensional quaternion principal component analysis (G2DQPCA) is presented as is a general framework that offers the great flexibility to fit various real-world applications.
To utilize the first-order condition of convex quaternion function, a new definition is proposed for the derivative of the quaternion norm.   
   A closed-form solution is obtained at each step of iteration.

\item  A new weighted G2DQPCA (WG2DQPCA) approach is proposed for color face recognition.
The projection bases are mathematically required to be orthogonal to each other and are weighted by corresponding objective function value to enhance the role of main features in color image recognition.

\item
Based on three standard color face databases, the evaluation of proposed methods is implemented on color face recognition and color image reconstruction. Numerical results indicate that  WG2DQPCA performs better than the state-of-the-art 2DQPCA-based algorithms and four  prominent deep learning methods.

\end{itemize}

The paper is organized as follows. In Section \ref{s:preview}, the fundamental information of quaternion  matrix theory is recalled.
In Section \ref{s:G2DQPCA}, a generalized two-dimensional color principal component analysis with weighted projection approach is presented based on quaternion models and  a closed-form solution is derived. In Section \ref{s:application},  the WG2DQPCA and G2DQPCA approaches are proposed for face recognition and image reconstruction, respectively.
In Section \ref{s:experiments}, numerical experiments are conduct by applying the Georgia Tech face  database, Color FERET face database and the  Faces95 database. Finally, the conclusion is given in Section \ref{s:conclusion}.
\section{Preliminaries}\label{s:preview}
In this section, we recall the basic information of quaternions, quaternion vectors and quaternion matrices.

Let us firstly describe some notation.~Letters of regular font denote scalars, vectors and matrices in the real domain. Boldface letters  denote quaternions, quaternion vectors and quaternion matrices.~${\rm sign}(\cdot)$ denotes the sign function;~$|\cdot|$ denotes the absolute value; $w\circ v$ denotes the
Hadamard product, i.e., the element-wise product between two
vectors; $\|\cdot\|_{1}$,~$\|\cdot\|_{2}$,~and~$\|\cdot\|_{p}$ denote $L_{1}$-norm,
$L_{2}$-norm, and $L_{p}$-norm, respectively.

Let  $\iq,~\jq,~\kq$ be three imaginary units satisfying
   \begin{equation*}\label{e1}
   \iq^2=\jq^2=\kq^2=\iq\jq\kq=-1
\end{equation*}
and let 
$$\Q\!=\!\left\{\aq=a^{(0)}+a^{(1)}\iq+a^{(2)}\jq+a^{(3)}\kq~|~a^{(0)},\cdots,a^{(3)}\in\R\right\}$$
 denote the quaternion skew-field.
\begin{definition}
For a quaternion $\aq\in\Q$  and a positive real number $\alpha\in\R$,  the $s$-absolute value of $\aq$  is defined by
$$|\aq|_\alpha\equiv \sqrt[\alpha]{|a^{(0)}|^\alpha+|a^{(1)}|^\alpha+|a^{(2)}|^\alpha+|a^{(3)}|^\alpha}.$$
\end{definition}
\noindent
For illustration, we plot the set of purely quaternion numbers satisfying $|\aq|_\alpha\le 1$ in Figure \ref{f:Lpnorm}.
If  $\alpha=2$, then $|\aq|_\alpha$ is exactly the absolute of $\aq$,  shortly denoted by  $|\aq|$.
The sign of a quaternion $\aq$ is defined by
\begin{equation*}
  \setlength{\nulldelimiterspace}{0pt}
  {\rm sign}(\aq)\equiv\left\{\begin{IEEEeqnarraybox}[\relax][c]{l's}
  \aq/|\aq|,&if $\aq\neq 0,$\\
  0,&if $\aq=0.$%
  \end{IEEEeqnarraybox}\right.
\end{equation*}
An $m\times n$ quaternion matrix is of the form $\Aq=A^{(0)}+A^{(1)}\iq+A^{(2)}\jq+A^{(3)}\kq$,  where $A^{(0)},\cdots, A^{(3)} \in \mathbb{R}^{m\times n}$.  A quaternion is a {\it pure quaternion} if its real part is zero. A  {\it pure quaternion matrix } is a matrix  whose elements are pure quaternions~$(A^{(0)}=0)$~or zero.  In the RGB color space, a  pixel can be represented with a pure quaternion,  $r\mathbf{i}+g\mathbf{j}+b\mathbf{k}$,  where  $r,~g,~b$  stand for the values of Red, Green and Blue components, respectively.  An $m\times n$ color image can be saved as an $m\times n$  pure quaternion matrix, $\Aq=[\aq_{ij}]_{m\times n}$,  in which each entry,  $\aq_{ij}=r_{ij}\iq+g_{ij}\jq+b_{ij}\kq$,  denotes one color pixel, and  $r_{ij}$, $g_{ij}$ and $b_{ij}$ are nonnegative integers \cite{jnw19,pcd03}.

\begin{figure}[!t]
  \centering
    \setlength{\abovecaptionskip}{0.0cm}
  \setlength{\belowcaptionskip}{-0.5cm}
 \includegraphics[width=0.45\textwidth,height=0.3\textwidth]{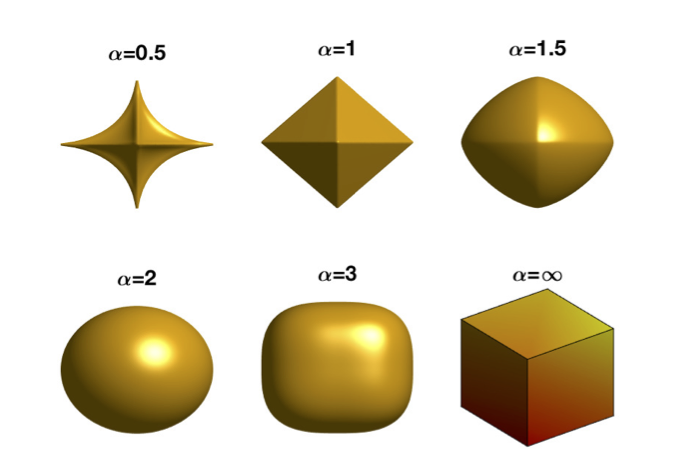}
\caption{Purely imaginary quaternions with $s$-absolute values equal to or less than $1$  ($|\aq|_{\alpha}\le1$).
}
\label{f:Lpnorm}
\end{figure}

\subsection{Norms of  quaternion vectors}
An $n$-dimensional quaternion vector is of the form
 \begin{equation*}\label{d:wq}
 \begin{aligned}
   \wq &=[\wq_{i}]_n =\:w^{(0)}+w^{(1)}\iq+w^{(2)}\jq+w^{(3)}\kq,
 \end{aligned}
\end{equation*}
where $\wq_{i}\!\in\Q$ denotes  the $i$-th entry  of $\wq$, and $w^{(0)}$, $\cdots$, $w^{(3)}\in\R^n$.
The sign and the absolute value of a quaternion vector are defined in the element-wise manner, that is
 \begin{IEEEeqnarray}{rCl}\label{e:abssign}
   |\wq|&\equiv&[\setlength{\arraycolsep}{0.2em} \begin{array}{c}|\wq_{i}|\end{array}\setlength{\arraycolsep}{5pt}]_{n}\IEEEyesnumber\IEEEyessubnumber\label{e:absv},\\
   {\rm sign}(\wq)&\equiv&[\setlength{\arraycolsep}{0.2em} \begin{array}{c}  {\rm sign}(\wq_{i})\end{array}\setlength{\arraycolsep}{5pt}]_{n}.\IEEEyessubnumber\label{e:signv}
 \end{IEEEeqnarray}
\begin{definition}
Let $\alpha$ and $p$ be two positive real numbers,   the $L_{\alpha,p}$-norm of  a quaternion vector $\wq\in\Q^n$ is defined by
\begin{equation*}\label{e:Lspnorm}
  \|\wq\|_{\alpha,p}\equiv\left(\sum\limits_{i=1}^{n}|\wq_{i}|_\alpha^{p}\right)^{\frac{1}{p}}.
\end{equation*}
\end{definition}
If $\alpha=p=2$,  the $L_{\alpha,p}$-norm reduces to the $L_2$-norm.
Because the $\alpha$-absolute value  quaternions have different geometric properties, people expect to apply different $L_{\alpha,p}$-norms in the regularization constrains  to  obtain  especial  features of color images.
In the following text, we set $\alpha$  as the default value $2$ to clarify the reasonability of the theory,  and  drop $\alpha$ from the subscript for simplicity.  That means the $L_{p}$-norm of  quaternion vector $\wq$ is denoted by
\begin{equation}\label{e:Lpnorm}
  \|\wq\|_{p}=\left(\sum\limits_{i=1}^{n}|\wq_{i}|^{p}\right)^{\frac{1}{p}}.
\end{equation}

\subsection{The real structure-preserving method}
Since the computation of quaternion matrices usually costs a lot of CPU times, we always apply the structure-preserving method (see \cite{jwl13}, \cite{jwzc18}, etc.) to simulate the calculation of quaternions only by real operations.

For quaternion matrix  $\Aq=A^{(0)}+A^{(1)}\iq+A^{(2)}\jq+A^{(3)}\kq\in\Q^{m\times n}$ and quaternion vector  $\wq=w^{(0)}+w^{(1)}\iq+w^{(2)}\jq+w^{(3)}\kq\in\Q^{n}$, we define
their real representations by
\begin{equation}\label{real-counterpart}
  \Aq^{(\Upsilon)}\equiv\left[
                 \begin{array}{rrrr}
                   A^{(0)} & -A^{(1)} & -A^{(2)} & -A^{(3)} \\
                   A^{(1)} & A^{(0)} & -A^{(3)} & A^{(2)} \\
                   A^{(2)} &A^{(3)}& A^{(0)} & -A^{(1)} \\
                   A^{(3)} & -A^{(2)} & A^{(1)} &A^{(0)} \\
                 \end{array}
               \right]
\end{equation}
and
\begin{equation}\label{d:real-vec}
  \wq^{(\gamma)}\!\equiv\!\left[
           \setlength{\arraycolsep}{0.2em}
            \begin{array}{llll}
              (w^{(0)} )^T&(w^{(1)})^T&(w^{(2)})^T&(w^{(3)})^T
            \end{array}
            \setlength{\arraycolsep}{5pt}
          \right]^T,
\end{equation}
respectively.
Note that   $\wq^{(\gamma)}$ is a $4n$-dimensional real vector. For convenience,
the absolute value, sign  and $L_{2,p}$-norm of $\wq^{(\gamma)}$ are defined by
 \begin{subequations}\label{signQ}
  \begin{align}
\label{e:absQ}  &|\wq^{(\gamma)}|_Q \equiv (|\wq|+|\wq|\iq+|\wq|\jq+|\wq|\kq)^{(\gamma)},\\
\label{e:signQ}                    &{\rm signQ}(\wq^{(\gamma)})\equiv({\rm sign}(\wq))^{(\gamma)},\\
&\|\wq^{(\gamma)}\|_{2,p}
  \equiv
  \left[\sum\limits_{i=1}^{n}\left(\sqrt{\sum\limits_{k=0}^3(\wq_{kn+i}^{(\gamma)})^{2}}\right)^{p}\right]^{\frac{1}{p}}.
  \end{align}
  \end{subequations}
It's  easy to verify that $\|\wq^{(\gamma)}\|_{2,p}$ is a vector norm on the real vector space if $p\geq1$.  If $0<p<1$, $\|\wq^{(\gamma)}\|_{2,p}$ is regarded as a non-convex and non-Lipschitz continuous function.
The $L_p$-norm of quaternion vector $\wq$ can be equivalently  written as
\begin{equation*}\label{def:norm_2,p}
\begin{aligned}
  \|\wq\|_{p}^{p}&=\sum\limits_{i=1}^{n}|\wq_{i}|^{p}
  =\sum\limits_{i=1}^{n}\left(\sqrt{\sum\limits_{j=0}^{3}(w_{i}^{(j)})^{2}}\right)^{p}= \|\wq^{(\gamma)}\|_{2,p}^{p}.
  \end{aligned}
\end{equation*}

 The definitions in \eqref{signQ} make it convenient to convert between quaternion vector norm and real vector norm equivalently.

\section{Generalized 2DQPCA}\label{s:G2DQPCA}
\noindent
In this section,  we present a new generalized 2DQPCA (G2DQPCA)  and quaternion optimization algorithms.

\subsection{G2DQPCA}
Suppose that  $\Fq_1,\Fq_2,...,\Fq_\ell\in\mathbb{Q}^{m\times n}$ are $\ell$ training samples of the quaternion matrix form and their mean  is
 \begin{equation*}
 {\bf\Psi}=\frac{1}{\ell}\sum\limits_{i=1}^{\ell}\Fq_i\in\H^{m\times n}.
 \end{equation*}
In the new G2DQPCA, the features are extracted  by  solving a quaternion optimization model with  $L_{s}$-norm in the objective function and $L_{p}$-norm in constraints:
\begin{equation}\label{e:g2dpcaQ}
\left.
  \begin{array}{c}
     (\widehat{\wq}_1,\cdots,\widehat{\wq}_k)\!=\!\!\!\!\!\mathop{{\rm arg~max}}\limits_{\wq_{1},\cdots,\wq_{k}\in\mathbb{Q}^{n}}\!\sum\limits_{j=1}^{k}\sum\limits_{i=1}^{\ell}\|(\Fq_{i}-{\bf\Psi})\wq_{j}\|_{s}^{s}, \\
    s.t.~\|\wq_{j}\|_{p}^{p}=1,~\wq_{j}^{*}\wq_{i}=0 , i\neq j, i,j=1,\cdots,k, \\
  \end{array}
\right.
\end{equation}
where $s\geq 1,p>0.$
The constraint  $\|\wq_{j}\|_{p}^{p}=1$ is convex if $p\ge 1$ and non-convex if $0<p<1$.    The constraint set is either convex or non-convex depending on the  value of $p$.

Now we are concerned with  the solvability of the optimization problem $\eqref{e:g2dpcaQ}$. As  primary results,  several necessary inequalities are derived  for  a linear optimization problem with $L_{p}$-norm in quaternion domain.  Their proofs are based on the first-order condition of convex function \cite{sl09}.
Let $w,v\in\R^{n}$ be two real vectors and let  $f(w)$   be a convex and differentiable a real function.
Then
\begin{equation}\label{convex-function}
  f(w)\geq f(v)+\nabla f(v)^{T}(w-v),
\end{equation}
where $\nabla$ denotes the gradient operator, and  the equality holds when $w = v$.
\begin{theorem}\label{lemma-1}
  Let $\wq,\vq\in\H^{n}$ be two quaternion vectors and $p\geq 1$, then there is
\begin{equation}\label{e:convex}
\begin{array}{lcl}
  \|\wq\|_p^{p}&\;\geq\;& p[|\vq^{(\gamma)}|_Q^{p-1}\circ {\rm signQ}(\vq^{(\gamma)})]^{T}\wq^{(\gamma)}\\
  &&+(1-p)\|\vq\|_{p}^{p}
  \end{array}
\end{equation}
 and the equality holds when $\wq=\vq$.
\end{theorem}
\begin{proof}
Denote  $\wq=[\wq_1,\cdots,\wq_n]^T$, ~$\vq=[\vq_1,\cdots,\vq_n]^T$, ~$\wq_i,\vq_i\in\Q$.  By the real structure-preserving method, we only need to prove
\begin{equation*}
\begin{aligned}
  \|\wq^{(\gamma)}\|_{2,p}^{p}&\geq p[|\vq^{(\gamma)}|_Q^{p-1}\circ {\rm signQ}(\vq^{(\gamma)})]^{T}\wq^{(\gamma)}\\
  &+(1-p)\|\vq^{(\gamma)}\|_{2,p}^{p}.
  \end{aligned}
\end{equation*}
  Firstly, we assume all entries of $\vq$ are non-zero elements. Then $\|\wq^{(\gamma)}\|_{2,p}^{p}$ is differentiable at $\wq^{(\gamma)}=\vq^{(\gamma)}$, ~and
  \begin{equation*}
    \frac{\partial\|\wq^{(\gamma)}\|_{2,p}^{p}}{\partial \wq^{(\gamma)}}=\left[
                                                         \begin{array}{cccc}
                                                           y_{0}^T&
                                                           y_{1}^T &
                                                           y_{2}^T &
                                                           y_{3}^T
                                                         \end{array}
                                                       \right]^T,
   \end{equation*}
 where for $j=0,1,2,3$,
 $$\  y_j = \left[
\begin{array}{cccc}
p\frac{|\wq_{1}|^{p-1}}{|\wq_{1}|}w_{1}^{(j)} &
   p\frac{|\wq_{2}|^{p-1}}{|\wq_{2}|}w_{2}^{(j)} &
   \cdots&
   p\frac{|\wq_{n}|^{p-1}}{|\wq_{n}|}w_{n}^{(j)}
          \end{array}
     \right]^T.$$\\
Applying the definitions in \eqref{signQ},~we obtain
\begin{equation}\label{diff-v}
\begin{aligned}
 \nabla\|\vq^{(\gamma)}\|_{2,p}^{p}&\equiv \frac{\partial\|\wq^{(\gamma)}\|_{2,p}^{p}}{\partial \wq^{(\gamma)}}\left|_{\wq=\vq}\right.\\
 &= p\cdot |\vq^{(\gamma)}|_Q^{p-1}\circ {\rm signQ}(\vq^{(\gamma)}).
\end{aligned}
\end{equation}
Clearly,
\begin{equation}\label{simple_vr}
\begin{aligned}
  |\vq^{(\gamma)}|_Q^{p-1}\circ {\rm signQ}(\vq^{(\gamma)}))^{T}\vq^{(\gamma)}
  \!=\!\|\vq^{(\gamma)}\|_{2,p}^{p}.
\end{aligned}
\end{equation}
Additionally, $\|\wq^{(\gamma)}\|_{2,p}^{p}$  is   a convex function because  $p \geq 1$. Together with equality  \eqref{diff-v}, applying the first-order convexity condition \eqref{convex-function} to the function $f(\cdot)=\|\cdot\|_{2,p}^p$ yields the desired result
\begin{equation}\label{e:lemma1}
\begin{aligned}
  &\|\wq^{(\gamma)}\|_{2,p}^{p} \\
 &\geq \|\vq^{(\gamma)}\|_{2,p}^{p}+(\nabla\|\vq^{(\gamma)}\|_{2,p}^{p})^{T}(\wq^{(\gamma)}-\vq^{(\gamma)})\\
  &=p[|\vq^{(\gamma)}|_Q^{p-1}\circ {\rm signQ}(\vq^{(\gamma)})]^{T}\wq^{(\gamma)}\\
 &\quad +(1-p)\|\vq^{(\gamma)}\|_{2,p}^{p}
  \end{aligned}
\end{equation}
where the  equality holds  when $\wq=\vq$.

Now we consider the case that    $\vq$ has zero entries.  For any two quaternions $\wq_i$ and $\vq_i$, there is
\begin{equation}\label{e:element}
\begin{aligned}
|\wq_{i}|^{p}
&\geq p\cdot|\vq_{i}|^{p-1}\cdot({\rm signQ}({\vq_{i}^{(\gamma)}}))^{T}\wq_{i}^{(\gamma)}\\
&+(1-p)|\vq_{i}|^{p}.
\end{aligned}
\end{equation}
If $\vq_{i}\neq 0$, this inequality is exactly \eqref{e:lemma1} being applied to two quaternions, because
$|\wq_i|^{p}=\left(\sqrt{\sum_{j=0}^{3}(w_{i}^{(j)})^{2}}\right)^{p}$
is convex and differentiable at $\wq_{i} = \vq_{i}$.
If $\vq_{i}= 0$, the inequality  \eqref{e:element}   reduces to $|\wq_{i}|^p \geq 0$, and  surely holds since ${\rm signQ}({\vq_{i}^{(\gamma)}})=0$ according the definition \eqref{e:signQ} (and under the assumption  $0^0=1$).
The equality in \eqref{e:element} holds  when $\wq_{i} = \vq_{i}$.
 Summing up \eqref{e:element} together, we obtain
\begin{equation*}
\begin{aligned}
&\sum\limits_{i=1}^{n}\left(\sqrt{\sum\limits_{j=0}^{3}(w_{i}^{(j)})^{2}}\right)^{p}\\
&\geq p\sum\limits_{i=1}^{n}\left[|\vq_{i}|^{p-1}\cdot({\rm signQ}({\vq_{i}^{(\gamma)}}))^{T}\wq_{i}^{(\gamma)}\right]\\
&+(1-p)\sum\limits_{i=1}^{n}\left(\sqrt{\sum\limits_{j=0}^{3}(v_{i}^{(j)})^{2}}\right)^{p}.
\end{aligned}
\end{equation*}

Consequently, \eqref{e:convex} holds and the inequality becomes equality when $\wq = \vq$ no matter $\vq$ has zero
entries or not. This completes the proof.
\end{proof}

Let $w\in\R^{n}$, $v\in\R^{n}$, and $p,q\in[1,\infty]$ be two scalars with $1/p+1/q=1$.
Then the H$\ddot{o}$lder's inequality\cite{WH91} states that
\begin{equation}\label{holder}
  \sum\limits_{i=1}^{n}|v_{i}w_{i}|\leq\|v\|_{q}\|w\|_{p}.
\end{equation}
The equality holds if and only if there exists a positive real
scalar $c$ satisfying $|w_{i}|^{p} = c |v_{i}|^{q}, i = 1, 2,\cdots,n$.
Since two arbitrary quaternions are not comparable, here we consider the linear optimization problem with $L_{p}$-norm
based on the H$\ddot{o}$lder's inequality.
\begin{lemma}\label{lemma-2}
Let $\wq\in\H^{n},\vq\in\H^{n}/\{0\}$, and let $p,q\in[1,\infty]$ be two scalars satisfying $1/p+1/q=1$.  Then the quaternion optimization problem
\begin{equation}\label{opt-p}
  \mathop{{\rm max}}\limits_{\wq}{\rm }\vq^{*}\wq, ~s.t.~\|\wq\|_{p}^{p}=1
\end{equation}
has a closed-form solution
\begin{equation*}
  \wq^{(\gamma)} = \frac{|\vq^{(\gamma)}|_Q^{q-1}}{\|\vq^{(\gamma)}\|_{2,q}^{q-1}}\circ{\rm signQ}(\vq^{(\gamma)}).
\end{equation*}
\end{lemma}
\begin{proof}
  According to the representation \eqref{d:real-vec}, the quaternion optimization problem \eqref{opt-p} is equivalently rewritten to
  \begin{equation*}
      \mathop{{\rm max}}\limits_{\wq^{(\gamma)}}(\vq^{(\gamma)})^{T}\wq^{(\gamma)}, ~s.t.~\|\wq^{(\gamma)}\|_{2,p}^{p}=1.
  \end{equation*}
Based on the H$\ddot{o}$lder's inequality \eqref{holder}, we have
  \begin{equation*}
  \begin{aligned}
    &(\vq^{(\gamma)})^{T}\wq^{(\gamma)}\leq \sum_{i=1}^{n}\sum_{j=0}^{3}|v_{i}^{(j)}w_{i}^{(j)}|\leq
    \sum\limits_{i=1}^{n}|\vq_{i}||\wq_{i}|=|\vq|^T|\wq|\\
    &\leq\|~|\vq|~\|_{q}\|~|\wq|~\|_{p} =\|\vq^{(\gamma)}\|_{2,q}\|\wq^{(\gamma)}\|_{2,p}
             =\|\vq^{(\gamma)}\|_{2,q}.
  \end{aligned}
  \end{equation*}
Therefore, the maximum of the objective function is obtained when the inequalities become equalities.
The equality in $(\vq^{(\gamma)})^{T}\wq^{(\gamma)}\leq \sum\limits_{i=1}^{n}\sum\limits_{j=0}^{3}|v_{i}^{(j)}w_{i}^{(j)}|$ holds when
\begin{equation*}
  {\rm sign}(v_{i}^{(j)})={\rm sign}(w_{i}^{(j)}).
\end{equation*}
The equality in $\sum\limits_{i=1}^{n}\sum\limits_{j=0}^{3}|v_{i}^{(j)}w_{i}^{(j)}|\leq
    \sum\limits_{i=1}^{n}|\vq_{i}||\wq_{i}|$
    holds when
\begin{equation}\label{condition:2}
  |w_{i}^{(j)}|^{2}=c_{i}|v_{i}^{(j)}|^{2},~j=0,1,2,3.
\end{equation}
 If $\vq_{i}=0$, we set  $c_i=0$ and $w_{i}^{(j)}=0$.
If $\vq_{i}\neq0$, then  the constant $c_{i}$ is calculated by
\begin{equation}\label{condition:c1}
  c_{i} = \frac{\sum\limits_{j=0}^{3}|w_{i}^{(j)}|^{2}}{\sum\limits_{j=0}^{3}|v_{i}^{(j)}|^{2}}=\frac{|\wq_{i}|^{2}}{|\vq_{i}|^{2}}.
\end{equation}
Inputting \eqref{condition:c1} into \eqref{condition:2}  yields the expression
\begin{equation}\label{wri}
\begin{aligned}
  |w_{i}^{(j)}|=\left(c_{i}|v_{i}^{(j)}|^{2}\right)^{\frac{1}{2}}
  &=\frac{|\wq_{i}|}{|\vq_{i}|}|v_{i}^{(j)}|.
\end{aligned}
\end{equation}
Then we have
\begin{equation}\label{wri2}
\begin{IEEEeqnarraybox}[\relax][c]{l's}
  w_{i}^{(j)}=\frac{|\wq_{i}|}{|\vq_{i}|}|v_{i}^{(j)}|{\rm sign}(v_{i}^{(j)}),&~{\rm if}~$\vq_i\neq 0$;\\
  ~w_{i}^{(j)}=0,&~{\rm if}~$\vq_i=0$.%
  \end{IEEEeqnarraybox}
\end{equation}

For simplicity, we introduce an  auxiliary  variable $Y^{\wq}$ to denote the absolute of quaternion vector $\wq$.
The equality in  $|\vq|^T|\wq|
    \leq\|~|\vq|~\|_{q}\|~|\wq|~\|_{p} $ holds when
\begin{equation}\label{wi:cvi}
  |Y_{i}^{\wq}|^{p} = c|Y_{i}^{\vq}|^{q},i=1,2,\cdots,n.
\end{equation}
Since $\vq\neq0$, the constant $c$ is then calculated by
\begin{equation}\label{constant:c}
  c 
  =\frac{\|Y^{\wq}\|_{p}^{p}}{\|Y^{\vq}\|_{q}^{q}}
  =\frac{\|\wq^{(\gamma)}\|_{2,p}^{p}}{\|\vq^{(\gamma)}\|_{2,q}^{q}}=\frac{1}{\|\vq^{(\gamma)}\|_{2,q}^{q}}.
\end{equation}
Substituting \eqref{constant:c} into \eqref{wi:cvi}, we have
\begin{equation*}
\begin{aligned}
  |Y_{i}^{\wq}|&=(c|Y_{i}^{\vq}|^{q})^{1/p}
 =\frac{|Y_{i}^{\vq}|^{q-1}}{\|Y^{\vq}\|_{q}^{q-1}}, i=1,2,\cdots,n,
\end{aligned}
\end{equation*}
i.e.
\begin{equation*}
  \|\wq_{i}\|_{2}=\frac{\|\vq_{i}\|_{2}^{q-1}}{\|\vq^{(\gamma)}\|_{2,q}^{q-1}}.
\end{equation*}
Together with equality \eqref{wri2},  we obtain the expression
\begin{equation*}
\begin{aligned}
  w_{i}^{(j)}&=\frac{\|\vq_{i}\|_{2}^{q-1}}{\|\vq^{(\gamma)}\|_{2,q}^{q-1}}\frac{1}{\|\vq_{i}\|_{2}}|v_{i}^{(j)}|{\rm sign}(v_{i}^{(j)})\\
  &=\frac{|\vq_{i}|^{q-1}}{\|\vq^{(\gamma)}\|_{2,q}^{q-1}}\frac{v_{i}^{(j)}}{|\vq_{i}|}.
  \end{aligned}
\end{equation*}
Connecting to the definitions in \eqref{signQ} and rewriting the equation into vector form,  we yields the desired result
\begin{equation*}
  \wq^{(\gamma)}=\frac{|\vq^{(\gamma)}|_Q^{q-1}}{\|\vq^{(\gamma)}\|_{2,q}^{q-1}}\circ{\rm signQ}(\vq^{(\gamma)}).
\end{equation*}
The proof is completed.
\end{proof}

\begin{lemma}\label{lemma-3}
  Let $\wq=[\wq_i],~\vq=[\vq_i]\in\H^{n}$ with  $\wq_i\neq0$ and $\vq_i\neq0$ for any $1\le i\le n$,  and  $0<p<1$.
 Then
\begin{equation*}
\|\wq\|_{p}^{p}\leq p(|\vq|^{p-1})^{T}|\wq|+(1-p)\|\vq\|_{p}^{p}
\end{equation*}
holds wherein the inequality becomes equality when $|\wq|=|\vq|$.
\end{lemma}
\begin{proof}
Recall that the  absolute value function is defined in $\eqref{e:absv}$.
For any quaternion vector $\wq$ with each entry  $\wq_i\neq0$,  $|\wq|$ is a real vector with positive entries. Let $f(|\wq|)=\||\wq|\|_p^p$, then $f(\cdot)$ is non-convex and differentiable at positive real vector  when $0<p<1$.  From the first-order convexity condition $\eqref{concave-function}$, we have
\begin{equation*}
f(|\wq|)\leq f(|\vq|)+\nabla f(|\vq|)^T(|\wq|-|\vq|),
\end{equation*}
and
\begin{equation*}
\begin{aligned}
\nabla f(|\vq|)=\nabla\||\vq|\|_p^p&=\frac{\partial\||\vq|\|_p^p}{\partial|\vq|}\Big|_{\wq=\vq}\\
&=p|\vq|^{p-1}\circ {\rm sign}(|\vq|).
\end{aligned}
\end{equation*}
Therefore,
\begin{equation*}
f(|\wq|)\leq\||\vq|\|_p^p+[p|\vq|^{p-1}\circ {\rm sign}(|\vq|)]^T|\wq|-p\||\vq|\|_p^p.
\end{equation*}
Thus, we obtain
$\||\wq|\|_{p}^{p}\leq p[|\vq|^{p-1}\circ {\rm sign}(|\vq|)]^{T}|\wq|+(1-p)\||\vq|\|_{p}^{p}.$
 Apparently, if $\wq^{(\gamma)}=\vq^{(\gamma)}$, it must satisfy the condition that $|\wq|=|\vq|.$
\end{proof}

Note that in our method, in order to meet the condition that $\wq_i\neq0$, if any element in the projection vector $(\wq^{(\gamma)})^{(k)}$ is zero, then we replace it with $(\wq^{(\gamma)})^{(k)}+\varepsilon$, ~where $\varepsilon$ is a random scalar that is sufficiently close to zero. 
In the rest of this section, we always assume that the mean of the training set is zero, that is, ${\bf\Psi}=0$; otherwise, we centralize the training samples by $\Fq_i=\Fq_i-{\bf\Psi} $, $i=1,\ldots, \ell$.

\subsubsection{Case 1: Convex constraint set}
Now, we proceed to the solution of G2DQPCA problem for the first case,~i.e.,~$p\geq1,$ based on Theorem $\ref{lemma-1}$ and Lemma $\ref{lemma-2}$.

 The quaternion optimization problem of G2DQPCA states
\begin{equation}\label{case-1}
   \mathop{{\rm max}}\limits_{\wq}\sum_{i=1}^{\ell}\|\Fq_{i}\wq\|_{s}^{s}, ~s.t.~ \|\wq\|_{p}^{p}=1,
\end{equation}
where $s\geq1,p\geq1,\wq\in\H^{n}$.  The constraint set is convex.
 According to definitions \eqref{real-counterpart} and \eqref{d:real-vec},  the quaternion optimization problem \eqref{case-1} can be transformed to its equivalent real counterpart,
\begin{equation}\label{real:case1}
    \mathop{{\rm max}}\limits_{\wq^{(\gamma)}}\sum_{i=1}^{\ell} \|\Fq_{i}^{(\gamma)}\wq^{(\gamma)}\|_{2,s}^{s}, ~s.t.~ \|\wq^{(\gamma)}\|_{p}^{p}=1.
\end{equation}
This real optimization problem can be turned into iteratively maximizing a surrogate function under
the MM framework, as shown below. Assume $(\wq^{(\gamma)})^{(k)}$ is the projection vector at the $k$-th step in the iteration procedure. It can be regarded as a constant vector that is irrelevant with respect
to $\wq^{(\gamma)}$.
Define
$$\hat{F_{i}}=|\Fq_{i}^{(\gamma)}(\wq^{(\gamma)})^{(k)}|_Q^{s-1}\circ{\rm signQ}(\Fq_{i}^{(\gamma)}(\wq^{(\gamma)})^{(k)}),$$
according to Theorem $\ref{lemma-1}$. The convex objective function
can be linearized as
\begin{equation*}
  \begin{aligned}
    &\sum\limits_{i=1}^{\ell}\|\Fq_{i}\wq\|_{s}^{s}=\sum\limits_{i=1}^{\ell}\|\Fq_{i}^{(\gamma)}\wq^{(\gamma)}\|_{2,s}^{s}\\
    &\geq s\sum\limits_{i=1}^{\ell}\hat{F_{i}}^{T}\Fq_{i}^{(\gamma)}\wq^{(\gamma)}+(1-s)\sum\limits_{i=1}^{\ell}\|\Fq_{i}^{(\gamma)}(\wq^{(\gamma)})^{(k)}\|_{2,s}^{s},
  \end{aligned}
\end{equation*}
wherein the inequality becomes equality when $\wq^{(\gamma)} = (\wq^{(\gamma)})^{(k)}$. Denote the objective function  by $f(\wq^{(\gamma)})$, and  the linearized function by $g(\wq^{(\gamma)}|(\wq^{(\gamma)})^{(k)})$.
That is
\begin{equation*}
  f(\wq^{(\gamma)})=\sum\limits_{i=1}^{\ell}\|\Fq_{i}^{(\gamma)}\wq^{(\gamma)}\|_{2,s}^{s},
\end{equation*}
and
\begin{equation}\label{gw|wk}
\begin{aligned}
  g(\wq^{(\gamma)}|(\wq^{(\gamma)})^{(k)})= s\sum\limits_{i=1}^{\ell}\hat{F_{i}}^{T}\Fq_{i}^{(\gamma)}\wq^{(\gamma)}\\
    +(1-s)\sum\limits_{i=1}^{\ell}\|\Fq_{i}^{(\gamma)}(\wq^{(\gamma)})^{(k)}\|_{2,s}^{s}.
  \end{aligned}
\end{equation}
According to $\eqref{simple_vr}$ and by simple algebra, it is easy to verify that
$$\hat{F_{i}}^{T}\Fq_{i}^{(\gamma)}(\wq^{(\gamma)})^{(k)}
=\|\Fq_{i}^{(\gamma)}(\wq^{(\gamma)})^{(k)}\|_{2,s}^{s}.$$
Then we have $f((\wq^{(\gamma)})^{(k)})= g((\wq^{(\gamma)})^{(k)}|(\wq^{(\gamma)})^{(k)})$
 and $f(\wq^{(\gamma)}) \geq g(\wq^{(\gamma)}|(\wq^{(\gamma)})^{(k)})$ for all $\wq^{(\gamma)}$, satisfying the two key conditions of the minorization-maximization (MM) framework \cite{dk04}.
Therefore, $g(\wq^{(\gamma)}|(\wq^{(\gamma)})^{(k)})$ is a feasible surrogate function of $f(\wq^{(\gamma)})$.
According to the MM framework, the optimization problem in \eqref{case-1} can be turned into iteratively maximizing the surrogate
function as follows
\begin{equation*}
\left.
  \begin{array}{l}
  (\wq^{(\gamma)})^{(k+1)}=\mathop{\rm arg~max}\limits_{\wq^{(\gamma)}}g(\wq^{(\gamma)}|(\wq^{(\gamma)})^{(k)}),\\
  s.t.~\|\wq^{(\gamma)}\|_{2,p}^{p}=1.
  \end{array}
  \right.
\end{equation*}
Define
\begin{equation}\label{vk}
  (\vq^{(\gamma)})^{(k)} = \sum\limits_{i=1}^{\ell}(\Fq_{i}^{(\gamma)})^{T}\hat{F_{i}},
\end{equation}
by dropping the term irrelevant to $\wq^{(\gamma)}$ in the surrogate function $\eqref{gw|wk}$,
maximizing the surrogate function leads to a linear optimization problem with $L_{p}$-norm constraint
\begin{equation}\label{linearcase1}
\begin{array}{l}
\begin{aligned}
  (\wq^{(\gamma)})^{(k+1)}
  =\mathop{\rm arg~max}\limits_{\wq^{(\gamma)}}[(\vq^{(\gamma)})^{(k)}]^{T}\wq^{(\gamma)},
  \end{aligned} \\
   s.t.~\|\wq^{(\gamma)}\|_{2,p}^{p}=1.
\end{array}
\end{equation}
According to Lemma $\ref{lemma-2}$, the solution of this problem is
\begin{equation}\label{case1-solution}
  (\wq^{(\gamma)})^{(k+1)}=\frac{|(\vq^{(\gamma)})^{(k)}|_Q^{q-1}}{\|(\vq^{(\gamma)})^{(k)}\|_{2,q}^{q-1}}\circ{\rm signQ}((\vq^{(\gamma)})^{(k)}),
\end{equation}
where $q$ satisfies $\frac{1}{p} + \frac{1}{q} = 1$.
The solution is rewritten in a two-step procedure as
\begin{equation}\label{2step}
\begin{aligned}
  &(\uq^{(\gamma)})^{(k)}=|(\vq^{(\gamma)})^{(k)}|_Q^{q-1}\circ{\rm signQ}((\vq^{(\gamma)})^{(k)}),\\
  &(\wq^{(\gamma)})^{(k+1)}=\frac{(\uq^{(\gamma)})^{(k)}}{\|(\uq^{(\gamma)})^{(k)}\|_{2,p}}.
\end{aligned}
\end{equation}

Two extreme conditions of case 1,~i.e., $p=1$ and $p=\infty$ are discussed as follows.~When $p=1$,~we will have $q=\infty$ according to the relation $1/p+1/q=1$. Then the denominator of  $\eqref{case1-solution}$ becomes $\|(\vq^{(\gamma)})^{(k)}\|_{2,\infty}$.~Let $j=\mathop{\rm arg~max}\limits_{i=1,\cdots,n}|\vq^{(k)}_{i}|$,~i.e.,~$|\vq^{(k)} _{j}|$ is the largest value in $|\vq^{(k)}|$.~By taking the limit of $\eqref{case1-solution}$ we have
\begin{equation}\label{p=1}
  \wq_{i}^{(k+1)}=\left\{\begin{array}{lc}
                         {\rm sign}(\vq_{i}^{(k)}), & i=j, \\
                         0, & i\neq j,
                       \end{array}\right.
\end{equation}
for $i=1,2,\cdots,n$.~Similarly, when $p$ approaches infinity, the limit of $\eqref{case1-solution}$ is
\begin{equation}\label{p=inf}
  (\wq^{(\gamma)})^{(k+1)}={\rm signQ}((\vq^{(\gamma)})^{(k)}).
\end{equation}
\subsubsection{Case 2: Non-convex constraint set}
When  $0 < p < 1$, the constraint set becomes non-convex.  With introducing an  auxiliary  variable $Y^{\wq}\equiv|\wq|$,  the  $L_{p}$-norm
\begin{equation}\label{YW}
  \|\wq\|_{p}^{p}=\|Y^{\wq}\|_{p}^{p},
\end{equation}
and clearly $Y^{\wq}\in\mathbb{R}^{n}$,$Y^{\wq}\geq0$.

Here we try to apply the method of Lagrange multipliers, considering that the constraint set is non-convex and non-Lipschitz continuous. Maximizing
the optimization problem of function $\eqref{e:g2dpcaQ}$  equals to maximizing the
Lagrangian as follows
\begin{equation}\label{Larg}
  \mathop{\rm max}\limits_{\wq}\sum\limits_{i=1}^{\ell}\|\Fq_{i}\wq\|_{s}^{s}-\lambda(\|\wq\|_{p}^{p}-1),
\end{equation}
where $s \geq1$, $0 < p < 1$, $\lambda > 0$, $\wq\in\H^{n}$.

Same as in Case 1, we firstly translate this function according to $\eqref{real:case1}$ and $\eqref{YW}$ into
\begin{equation*}
  \mathop{\rm max}\limits_{\wq^{(\gamma)}}\sum\limits_{i=1}^{\ell}\|\Fq_{i}^{(\gamma)}\wq^{(\gamma)}\|_{2,s}^{s}-\lambda(\|Y^{\wq}\|_{p}^{p}-1),
\end{equation*}
where $\|Y^{\wq}\|_{p}^{p}=\|\wq^{(\gamma)}\|_{2,p}^{p}$ is
essentially a function related to $\wq^{(\gamma)}$.
Again, the problem is transformed into iteratively maximizing a surrogate function
under the MM framework.
Assume that ~$(\wq^{(\gamma)})^{(k)}$ ~is the projection vector at the $k$-th step in the iteration
procedure. If any element in $(\wq^{(\gamma)})^{(k)}$ is zero, then we replace it with $(\wq^{(\gamma)})^{(k)}+\varepsilon$ to make sure that it has no zero elements, where $\varepsilon$ is a random scalar that is sufficiently close to zero. According to Theorem $\ref{lemma-1}$ and Lemma $\ref{lemma-3}$, we have
\begin{equation*}\label{case2-ineq}
\begin{aligned}
  \sum\limits_{i=1}^{\ell}\|\Fq_{i}^{(\gamma)}\wq^{(\gamma)}\|_{2,s}^{s}-\lambda(\|Y^{\wq}\|_{p}^{p}-1)
  \geq s((\vq^{(\gamma)})^{(k)})^{T}\wq^{(\gamma)}\\
  +(1-s)\sum\limits_{i=1}^{\ell}\|\Fq_{i}^{(\gamma)}(\wq^{(\gamma)})^{(k)}\|_{2,s}^{s}-\lambda(1-p)\||\wq^{(k)}|\|_{p}^{p}\\
  -\lambda p(|\wq^{(k)}|^{p-1}\circ {\rm sign}(|\wq^{(k)}|))^{T}Y^{\wq}+\lambda,
\end{aligned}
\end{equation*}
where $(\vq^{(\gamma)})^{(k)}$ is defined in $\eqref{vk}$,~$$|\wq^{(k)}|=\left[
                                                                \begin{array}{cccc}
                                                                  |\wq_{1}^{(k)}| & |\wq_{2}^{(k)}| & \cdots & |\wq_{n}^{(k)}| \\
                                                                \end{array}
                                                              \right]^{T},$$
 and this inequality becomes equality when $\wq^{(\gamma)}=(\wq^{(\gamma)})^{(k)}$.
~Denote the left-hand side of the inequality by $f(\wq^{(\gamma)})$ and  the right-hand side  by $g(\wq^{(\gamma)}|(\wq^{(\gamma)})^{(k)})$, ~i.e.,
\begin{equation}\label{fw_case2}
  f(\wq^{(\gamma)})=\sum\limits_{i=1}^{\ell}\|\Fq_{i}^{(\gamma)}\wq^{(\gamma)}\|_{2,s}^{s}-\lambda(\|Y^{\wq}\|_{p}^{p}-1),
\end{equation}
and
\begin{equation*}\label{gw|wk-case2}
  \begin{aligned}
&    g(\wq^{(\gamma)}|(\wq^{(\gamma)})^{(k)})= s((\vq^{(\gamma)})^{(k)})^{T}\wq^{(\gamma)}\\
 &   +(1-s)\sum\limits_{i=1}^{\ell}\|\Fq_{i}^{(\gamma)}(\wq^{(\gamma)})^{(k)}\|_{2,s}^{s}-\lambda(1-p)\||\wq^{(k)}|\|_{p}^{p}\\
&  -\lambda p(|\wq^{(k)}|^{p-1}\circ {\rm sign}(|\wq^{(k)}|))^{T}Y^{\wq} +\lambda.
  \end{aligned}
\end{equation*}
By the simple algebraic calculation,~we obtain
$$f((\wq^{(\gamma)})^{(k)})= g((\wq^{(\gamma)})^{(k)}|(\wq^{(\gamma)})^{(k)})$$
 and
 $$f(\wq^{(\gamma)}) \geq g(\wq^{(\gamma)}|(\wq^{(\gamma)})^{(k)})$$
  for all $\wq^{(\gamma)}$,  which satisfy the two key conditions of the MM framwork. Therefore, $g(\wq^{(\gamma)}|(\wq^{(\gamma)})^{(k)})$ is a feasible surrogate function of $f(\wq^{(\gamma)})$.
According to the MM framework, the optimization problem in $\eqref{Larg}$ can be turned into iteratively maximizing the surrogate
function as follows
\begin{equation*}
  (\wq^{(\gamma)})^{(k+1)}=\mathop{\rm arg~max}\limits_{\wq^{(\gamma)}}g(\wq^{(\gamma)}|(\wq^{(\gamma)})^{(k)}).
\end{equation*}
Then the following quadratic optimization problem will be reached after ignoring irrelevant terms of $\wq^{(\gamma)}$
\begin{equation}\label{wk+1-case2}
\begin{aligned}
  (\wq^{(\gamma)})^{(k+1)}&=\mathop{\rm arg~max}\limits_{\wq^{(\gamma)}}s((\vq^{(\gamma)})^{(k)})^{T}\wq^{(\gamma)}\\
  &-\lambda p(|\wq^{(k)}|^{p-1}\circ {\rm sign}(|\wq^{(k)}|))^{T}Y^{\wq}.
\end{aligned}
\end{equation}
Let
$
 h(\wq^{(\gamma)})= s((\vq^{(\gamma)})^{(k)})^{T}\wq^{(\gamma)}-\lambda p(\hat{\wq}^{(k)})^{T}Y^{\wq},
$
where $\hat{\wq}^{(k)}=|\wq^{(k)}|^{p-1}\circ {\rm sign}(|\wq^{(k)}|)$ and $Y^{\wq}$ is defined in $\eqref{YW}$.~Now we consider the partial derivative with respect to $\wq^{(\gamma)}$
\begin{equation}\label{Div-wr}
\begin{aligned}
  \left[\frac{\partial h(\wq^{(\gamma)})}{\partial \wq^{(\gamma)}}\right]_{l}&=\frac{\partial h(\wq^{(\gamma)})}{\partial w_{i}^{(j)}}\\
  &=s(\vq^{(k)})_{i}^{(j)}-\lambda p\left[|\wq^{(k)}_{i}|^{p-1}\frac{w_{i}^{(j)}}{|\wq_{i}|}\right],
\end{aligned}
\end{equation}
where $i=1,2,\cdots,n$,~$j=0,1,2,3$ and $l=1,2,\cdots,4n$.\\
Let the above equation equal to zero and we get
\begin{equation*}
  w_{i}^{(j)}=\frac{s}{\lambda p}(\vq^{(k)})_{i}^{(j)}|\wq_{i}||\wq^{(k)}_{i}|^{1-p}.
\end{equation*}
Rewrite the component as vector expression which is the solution of the problem $\eqref{wk+1-case2}$
\begin{equation*}
  (\wq^{(\gamma)})^{(k+1)}=\frac{s}{\lambda p}(\vq^{(\gamma)})^{(k)}\circ \tilde{\wq}^{(\gamma)},
\end{equation*}
where $\tilde{\wq}^{(\gamma)}= |\wq^{(\gamma)}|_Q\circ |(\wq^{(\gamma)})^{(k)}|_Q^{1-p}.$
Considering the constraint $\|\wq\|_{p}^{p}=1$,~i.e.~$\|\wq^{(\gamma)}\|_{2,p}^{p}=1$ and $\lambda>0$, ~we have
\begin{equation*}
  \lambda=\frac{s}{p}\|(\vq^{(\gamma)})^{(k)}\circ \tilde{\wq}^{(\gamma)}\|_{2,p}.
\end{equation*}
Then the update rule is
\begin{equation}\label{case2-solution}
  (\wq^{(\gamma)})^{(k+1)}=\frac{(\vq^{(\gamma)})^{(k)}\circ \tilde{\wq}^{(\gamma)}}{\|(\vq^{(\gamma)})^{(k)}\circ \tilde{\wq}^{(\gamma)}\|_{2,p}}.
\end{equation}
The above solution equals to the two-step procedure as below
\begin{align}
  (\uq^{(\gamma)})^{(k)}&= (\vq^{(\gamma)})^{(k)}\circ \tilde{\wq}^{(\gamma)},\\
  (\wq^{(\gamma)})^{(k+1)}&=\frac{(\uq^{(\gamma)})^{(k)}}{\|(\uq^{(\gamma)})^{(k)}\|_{2,p}}.
\end{align}
This completes the solution in Case 2.

\bigskip
Now we have obtained  the solution of the quaternion optimization problem of G2DQPCA.
From the results in $\eqref{case1-solution}$ and $\eqref{case2-solution}$, we observe
that a closed-form solution is obtained in each iteration for both cases.

\begin{remark}
G2DQPCA is a generalization of 2DQPCA  with  applying $L_{p}$-norm both in the objective function and the constraint function. Let $s=p=2$ then G2DQPCA reduces to 2DQPCA \cite{jlz17:2DQPCA,  zjg19:im-2dqpca},
  \begin{equation}\label{o:2dpcaQ}
  \begin{aligned}
  &\mathop{{\rm arg~max}}\limits_{\wq_{1},\cdots,\wq_{r}\in\Q^n}\sum\limits_{j=1}^{r}\sum\limits_{i=1}^{\ell}\|\Fq_{i}\wq_{j}\|_{2}^{2},\\
  &s.t.~\left\{
         \begin{array}{ll}
           \wq_{j}^{*}\wq_{i}=\|\wq_{j}\|_{2}^{2}=1 & (i=j),\\
           \wq_{j}^{*}\wq_{i}=0 & (i\neq j).
         \end{array}
       \right.
\end{aligned}
\end{equation}
Indeed, let $\Wq=[\wq_1,\cdots,\wq_r]$ then
\begin{equation*}\label{o:2dpcaQequivalent}
  \begin{aligned}
  &\mathop{{\rm arg~max}}\limits_{\wq_{1},\cdots,\wq_{r}\in\Q^n}\sum\limits_{j=1}^{r}\sum\limits_{i=1}^{\ell}\|\Fq_{i}\wq_{j}\|_{2}^{2}=\mathop{{\rm arg~max}}\limits_{\Wq\in\Q^{n\times k}}\sum\limits_{i=1}^{\ell}\|\Fq_{i}\Wq\|_{F}^{2}\\
  &
  =\mathop{{\rm arg~min}}\limits_{\Wq\in\Q^{n\times k}}\sum\limits_{i=1}^{\ell}\|\Fq_{i}(I-\Wq\Wq^*)\|_{F}^{2}.
\end{aligned}
\end{equation*}
The  ridge regression model \cite{xz19:2DPCA-S} of 2DQPCA  is equivalent to the quaternion optimization model \eqref{o:2dpcaQ}.
Notice that  the Frobenius norm of  a quaternion matrix $\Aq$ is defined by $\|\Aq\|_F^2={\rm trace}(\Aq^*\Aq)={\rm trace}(\Aq\Aq^*)$.

\end{remark}
\begin{remark}
G2DQPCA generalizes  G2DPCA from the real field to the quaternion skew-field.
Compared with G2DPCA , we firstly constrain the orthogonality of the projection vectors,  $\wq_{i},i = 1,2,\cdots,r$, in the ridge regression model \eqref{e:g2dpcaQ}. We will elaborate this constraint later.

\end{remark}

\subsection{A new quaternion optimization algorithm with deflation}
Now assume we have obtained the first $r$ projection vectors,~i.e.,~$\Wq=[\wq_{1},\wq_{2},\cdots,\wq_{r}]$, where $1\leq r<n$.
The $(r+1)$-th projection vector $\wq_{r+1}$ can be calculated similarly on the deflated samples
\begin{equation}\label{deflated}
  \Fq^{\textit{deflated}}_{i}=\Fq_{i}(\Iq-\Wq\Wq^{*}) , i= 1,2,\cdots,\ell.
\end{equation}
What is particularly noteworthy is that the projection vector $\wq_{r+1}$ obtained at each iteration must be orthonormalized against all previous $\wq_{i},~i=1,2,\cdots,r$ ~by a standard Gram-Schmidt procedure in quaternion domain.
This is because when we completed deflation in the $r$-th direction that means there is no information left in this direction and then projection on the deflated samples should be zero.
We observe that after $r$ steps each sample is transformed into \eqref{deflated}. After deflating samples by $r$ directions, we obtain that
\begin{equation}\label{orthonormal}
  \Fq_{i}^{\textit{deflated}}\wq_j=\Fq_{i}(\Iq-\Wq\Wq^{*})\wq_{j},
\end{equation}
where $j=1,2,\cdots,r+1.$
$\Fq_{i}^{\textit{deflated}}\wq_j$ should be zero if $j=1,2,\cdots,r$. Otherwise the feature information can be lost because of the interference from other direction. From \eqref{orthonormal} with $j=r+1$, we also know that the $(r+1)$-th projector must not be linearly represented by the prior projectors, or in other words, the newly computed projector is not in the subspace generated by the known projectors, so that we orthogonalize the computed projector to the known ones.

\begin{algorithm}[H]
\caption{G2DQPCA: the generalized two-dimensional quaternion principal component analysis}
\label{G2DQPCA}
\begin{algorithmic}[1]
\REQUIRE {Training samples ~$\Fq_1, \Fq_2,\cdots, \Fq_{\ell}$, the number of selected feature $r$, and parameters $s\in [1,\infty), p\in(0,\infty]$. }
\ENSURE Optimal quaternion projection matrix $\Wq$, and weighted coefficient vector $\Dq $
\STATE Initialize $\Wq=[ \ ], \Dq=[ \ ], \Fq^{0}_{i}=\Fq_{i}.$
\FOR{$t=1,2,\cdots,r$}
\STATE Initialize $k=0,\delta=1$, arbitrary $\wq^{(0)}$ with $\parallel\wq^{(0)}\parallel_p=1$.
\STATE $f^{(0)}=\sum\limits^{\ell}_{i=1}\|\Fq_{i}\wq^{(0)}\|_s^s.$
\WHILE{$\delta >10^{-4}$}
\STATE $\vq^{(k)}=\sum\limits^{\ell}_{i=1}\Fq_{i}^{*}[|\Fq_{i}\wq^{(k)}|^{s-1}\circledcirc {\rm sign}(\Fq_i\wq^{(k)})]$.
\IF{$0<p<1$}
\STATE $\uq^{(k)}=|\wq^{(0)}|\circledcirc|\wq^{(k)}|^{1-p}\circledcirc\vq^{(k)},$
\STATE $\wq^{(k+1)}=\frac{\uq^{(k)}}{\parallel\uq^{(k)}\parallel_p}.$
\ELSIF{$p=1$}
\STATE $j=\mathop{\rm arg~max}\limits_{1\le i\le n}|\vq_{i}^{(k)}|,$
\STATE $ \wq_{i}^{(k+1)}=\left\{
\begin{aligned}
{\rm sign}(\vq_{j}^{(k)}), \ i=j,\\
0, \ i\neq j.\\
\end{aligned}
\right.$
\ELSIF{$1<p<\infty$}
\STATE $q=p/(p-1),$
\STATE $\uq^{(k)}=|\vq^{(k)}|^{q-1}\circledcirc{\rm sign}(\vq^{(k)}),$
\STATE $\wq^{(k+1)}=\frac{\uq^{(k)}}{\parallel\uq^{(k)}\parallel_p}.$
\ELSIF{$p=\infty$}
\STATE $\wq^{(k+1)}={\rm sign}(\vq^{(k)}).$
\ENDIF
\STATE $f^{(k+1)}=\sum\limits^{\ell}_{i=1}\|\Fq_{i}\wq^{(k+1)}\|_{s}^{s}.$
\STATE $\delta=|f^{(k+1)}-f^{(k)}|/|f^{(k)}|.$
\STATE $k\leftarrow k+1.$
\ENDWHILE
\STATE $\wq_{t}=\wq^{(k)}.$
\STATE Orthogonalize $\wq_{t}$ with the previous vector $\wq_{i}$ by quaternion QR factorization.
\STATE $f^{(t)}=\sum\limits^{\ell}_{i=1}\|\Fq_{i}\wq_{t}\|_{s}^{s}.$
\STATE $\Wq\leftarrow [\Wq,\wq_{t}].$
\STATE $\Dq\leftarrow [\Dq,f^{(t)}].$
\STATE $\Fq_i=\Fq_i^0(\Iq-\Wq\Wq^{*}), i=1,2,\cdots,\ell.$
\ENDFOR
\end{algorithmic}
\end{algorithm}
\begin{figure}[htbp]
\centering
 \includegraphics[width=0.45\textwidth,height=1.0\textwidth]{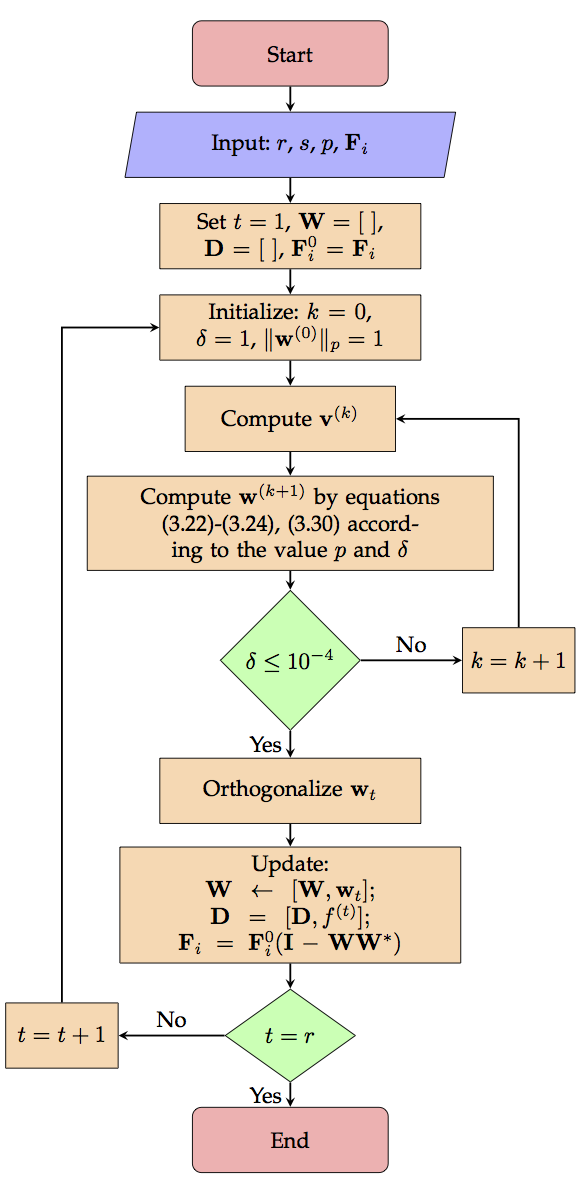}
\caption{Flowchart of Algorithm \ref{G2DQPCA}}\label{flowchart}
\end{figure}

Another notable feature here is the usage of deflation in general. It should not
be used to compute more than a few projection vectors because of the fact that the simple matrix $F_{i}$ will accumulate errors from all previous computations and this can be disastrous
if the cases where cancellations are so severe in the orthogonalization steps. In our experiments, we use quaternion QR Factorization\cite{jwzc18} to complete such orthogonalizing process.

We summarize the above steps into Algorithm \ref{G2DQPCA}. Here the notation $\circledcirc$ is defined by multiplying the corresponding real coefficients between two quaternions.~For example, assume $\wq=w_{0}+w_{1}\iq+w_{2}\jq+w_{3}\kq$,~$\vq=v_{0}+v_{1}\iq+v_{2}\jq+v_{3}\kq$,~then
\begin{equation*}
  \wq\circledcirc\vq=(w_{0}\circ v_{0})+(w_{1}\circ v_{1})\iq+(w_{2}\circ v_{2})\jq+(w_{3}\circ v_{3})\kq.
\end{equation*}
The flowchart of Algorithm \ref{G2DQPCA} is shown in Figure \ref{flowchart}.

\section{Color Image Analysis with Weighted Projection}\label{s:application}
In this section, we present the weighted G2DQPCA algorithm for color image recognition and the G2DQPCA algorithm for  the low-rank reconstructions of a group of color images.

With applying  Algorithm \ref{G2DQPCA} to the training data in color image recognition,  we obtain $r$ pairs of optimal values and projection vectors, denoted by $(f^{(1)}, \wq_1),\ldots, (f^{(r)}, \wq_r)$.
The contributions of $\wq_1, \ldots, \wq_r$ to increase the recognition rate can be characterized by $f^{(1)},\ldots,f^{(r)}$.
Let  $t=1,\ldots,r$.  It is not difficult to find out that when $s=p=2$, the value of $f^{(t)}$ is equal to the $t$-th eigenvalue of the covariance matrix of training set,  which represents the variance  on the direction $\wq_t$.
In practical implementation, we normalize  $f^{(t)}$  by $f^{(t)}/\sum_{t=1}^r f^{(t)}$  and  weight each projection vector $\wq_t$ by multiplying the normalized   $f^{(t)}$.

A new weighted G2DQPCA algorithm  can be proposed for color image recognition with using  the weighted projection vectors  $\wq_1f^{(1)}$, $\ldots$, $\wq_rf^{(r)}$. 
Let $\Wq=[\wq_1f^{(1)}$, $\cdots$, $\wq_rf^{(r)}]$ and let  the eigenface subspace be  generated by the columns of $\Wq$.
Then the weighted  projections of $\ell$ training face images on the eigenface subspace $\Wq$ are computed by
\begin{equation}\label{e:ps4fs}
\Pq_i=(\Fq_i-{\bf{\Psi})}\Wq\in\H^{m\times r},\ i=1,\cdots,\ell,
\end{equation}
where  ${\bf{\Psi}}$ is the average image of all training samples.
Such  $\Pq_i$ is called the {\it feature matrix} or {\it feature image} of the sample image $\Fq_i$
and its $j$-th column,  i.e., $(\Fq_i-{\bf{\Psi}})\wq_j$,  is called the {\it $j$-th principal component}.
 In Algorithm $\ref{WG2DQPCQ}$, we propose the procedure of G2DQPCA with  weighted projection for color image recognition.
 \begin{algorithm}[H]
 \caption{\bf G2DQPCA with Weighted Projection for Color  Image Recognition}
 \label{WG2DQPCQ}
   \begin{algorithmic}[1]
\STATE
 For the mean-centered given training samples $\Fq_{i}-{\bf\Psi},~i=1,2,\cdots,\ell$, compute the $r$ $(1\le r< n)$ projection vectors and their corresponding weighted coefficient by Algorithm $\ref{G2DQPCA}$,  denoted by $(\wq_1, f^{(1)}),$ $\ldots,$ $(\wq_r, f^{(r)})$.
\STATE Let the eigenface subspace
be spanned by the columns of  $\Wq=[f^{(1)}\wq_1,$ $\ldots,$ $f^{(r)}\wq_r]$.

\STATE Compute the projections $\Pq_1,\cdots,\Pq_r$, defined by \eqref{e:ps4fs}, of $\ell$ training color  images.
\STATE For a given testing sample $\Fq$,  compute its feature matrix, $\Pq=(\Fq-{\bf\Psi})\Wq$.
 Seek the nearest face image whose feature matrix $\Pq_{i^\#}$ $(1\le i^\#\le \ell)$ satisfies that
$i^\#=\mathop{\rm arg~min}\limits_{1\le i\le \ell}\|\Pq_i-\Pq\|$.  
Then $\Fq_{i^\#}$ is output as the person to be recognized.
   \end{algorithmic}
 \end{algorithm}

We close this section by applying G2DQPCA to compute the low-rank reconstructions of  color images. 
Let the eigenface subspace be spanned  by the columns of
 $\Wq=[\wq_1,$ $\ldots,$ $\wq_r]$,  where $\wq_{j}$ is computed by Algorithm $\ref{G2DQPCA}$. 
Then the  reconstructions of training samples are computed by
\begin{equation}\label{re-image}
  \Fq_{i}^{rec}=(\Fq_{i}-{\bf\Psi})\Wq\Wq^{*}+{\bf\Psi},~i=1,\cdots,\ell.
\end{equation}
That is, we only use the first $r$ projection vectors to reconstruct the original image.
One important thing to note here is that the reconstruction process does not need to magnify the effect of features, so the projection matrix $\Wq$ is unweighted.

One can see that $\Fq_{i}^{rec}-{\bf\Psi}$ is of low-rank  since  the rank $\Wq$  is $r$  and $r$ is always a small integer.  In fact, we have present a new method to simultaneously  compute the low-rank approximations to  $\Fq_{i}-{\bf\Psi}$, $i=1,\cdots,\ell$.  
The reconstruction of testing samples can be  computed in the similar way.
The whole  process is proposed in Algorithm $\ref{alg:reco}$.
\begin{algorithm}[H]
 \caption{\bf G2DQPCA for Image Reconstruction}
 \label{alg:reco}
   \begin{algorithmic}[1]
\STATE
For the mean-centered given training samples $\Fq_{i}-{\bf\Psi},~i=1,2,\cdots,\ell$, compute the $r$ $(1\le r< n)$ projection vectors by Algorithm $\ref{G2DQPCA}$, denoted by $\wq_1, \wq_2 \cdots,\wq_{r}$.
\STATE Let the eigenface subspace 
be spanned by the columns of $\Wq=[\wq_1,$ $\ldots,$ $\wq_r]$.
\STATE Compute the reconstructions $\Fq_{1}^{rec},\cdots, \Fq_{r}^{rec}$ of training samples by \eqref{re-image}.
\STATE  For a given testing sample $\Fq$, compute its reconstruction  by  
$
  \Fq^{rec}=(\Fq-{\bf\Psi})\Wq\Wq^{*}+{\bf\Psi}.
$
   \end{algorithmic}
 \end{algorithm}

\section{Experiments}\label{s:experiments}
\noindent

\begin{figure*}[!t]
    \setlength{\abovecaptionskip}{0.1cm}
  \setlength{\belowcaptionskip}{0.1cm}
  \centering
    \includegraphics[width=0.99\textwidth,height=0.4\textwidth]{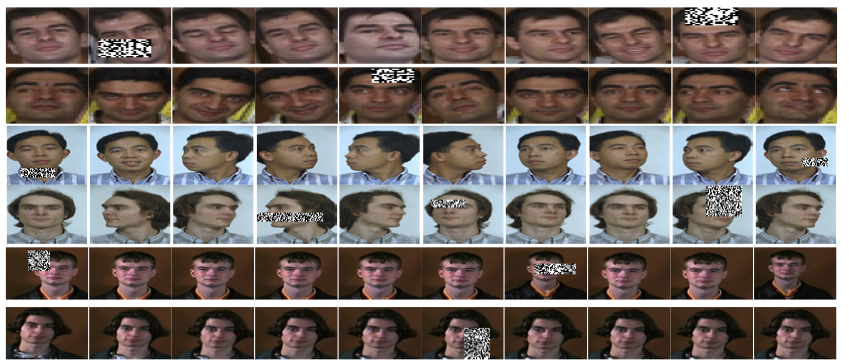}
\caption{Some samples with or without occlusion of one person from the GTFD, Color FERET and  Faces95 databases.}
\label{fig:samples}
\end{figure*}
In this section, we  evaluate the proposed G2DQPCA and WG2DQPCA models on color face recognition,  test the efficiency of G2DQPCA on color image reconstruction and compare WG2DQPCA with four well-known deep learning methods. 

The evaluation and comparison are implemented  on  three standard  face databases:
\begin{itemize}
\item GTFD---The Georgia Tech Face Database\footnote{The Georgia Tech face database: http://www.anefian.com /research / face\_reco.htm.} 
 contains 750 images from 50 subjects, fifth images per subject.
 All the images are manually  cropped, and then resized to $44 \times 33$ pixels.
\item Color FERET---The color FERET database\footnote{The color FERET database: https://www.nist.gov/itl/iad/image-group / color-feret-database.} 
 contains $1199$ persons, $14126$ color face images,  and each person has various numbers of face images with various backgrounds.
 The size of each cropped color face image is $192\times 128$ pixels. 
Here we take  $275$ individuals and $11$ images of each person for testing.
\item Faces95---Faces95 database\footnote{
The Faces95 database:  https://cswww.essex.ac.uk/mv /allfaces/ faces95.html.} 
contains $1440$ images photographed over a uniform background from $72$ subjects, $20$ images per subject. The size of each color face image is $200\times 180$ pixels.
\end{itemize}
 Some samples are listed in  Figure \ref{fig:samples}.

Notice that  the performance of 2DQPCA has been indicated to be better than  2DPCA in \cite{jlz17:2DQPCA,zjg19:im-2dqpca} and that  2DQPCA proposed in \cite{jlz17:2DQPCA,  zjg19:im-2dqpca,xz19:2DPCA-S} are  special cases of G2DQPCA with $s=p=2$. It is only necessary to concentrate on the comparison of G2DQPCA-based approaches with different choices of $s$ and $p$ and the efficiency of utilizing the weighting technique.
All numerical experiments are performed with MATLAB-R2016 on a personal computer with Intel(R) Xeon(R) CPU E5-2630 v3 @  2.4GHz (dual processor) and RAM 32GB.

\begin{example}[Color face recognition]\label{ex:1}
\begin{figure*}[!t]
  \setlength{\abovecaptionskip}{0.1cm}
  \setlength{\belowcaptionskip}{0.1cm}
  \centering
  \includegraphics[width=0.99\textwidth,height=0.4\textwidth]{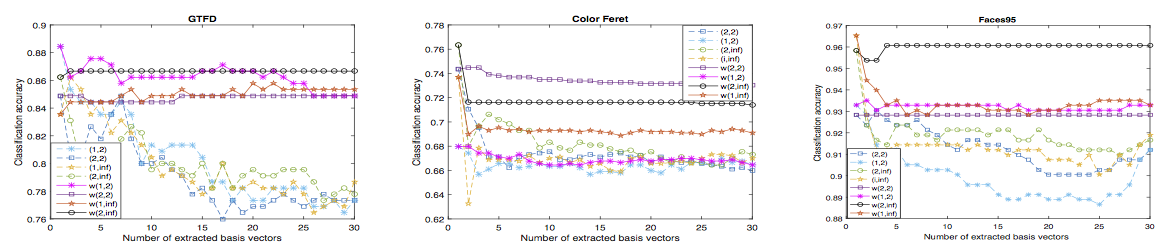}
\caption{Classification accuracies of G2DQPCA and WG2DQPCA on the GTFD, Color FERET and Faces95 databases.}
\label{fig:acc}
\end{figure*}

In this example, we  proceed to investigate the classification performance of G2DQPCA and WG2DQPCA on all three databases with clean training data. Randomly select 90 percent of data set as training set and the remaining as the testing set. The training set is guaranteed to contain at least one image of each individual.

The whole procedure is repeated three times and the average recognition rate is reported.
For simplicity, let $(s,p)$ denote G2DQPCA with $L_s$-norm and $L_p$-norm, and $w(s,p)$ WG2DQPCA with $L_s$-norm and $L_p$-norm in all figures.
For instance, $(2,2)$ represents G2DQPCA with $s=p=2$, which is exactly the version of 2DQPCA proposed in \cite{jlz17:2DQPCA,  zjg19:im-2dqpca,xz19:2DPCA-S};   $(1, 2)$, $(1, inf)$ and $(2, inf)$  denote  other new  versions of 2DQPCA, which are gathered in G2DQPCA. Note that $inf$ means $L_\infty$-norm.
  In Figure $\ref{fig:acc}$,  we present the classification accuracies of G2DQPCA and WG2DQPCA methods applied on three face databases with the number of features increases from $1$ to $30$.

From the numerical results, one can see that the classification accuracies of WG2DQPCA are much higher than those of G2DQPCA. The recognition rate of G2DQPCA decreases as the number of features increases, but the recognition rate of WG2DQPCA remains unchanged or perturbs tightly.
 The optimal parameter pair $(s,p)$ replies on the databases.
For the GTFD, Color FERET, and Faces95 databases, the optimal recognition rates are reached when  $w(s=1,p=2)$, $w(s=2,p=2)$ and $w(s=2,p=\infty)$, respectively.
This performance verifies that the G2DQPCA-based approaches with different $L_p$ norms extract quite distant features of color images.

Moreover,   2DQPCA proposed in \cite{jlz17:2DQPCA,  zjg19:im-2dqpca,xz19:2DPCA-S},  indicated by $(2,2)$ in Figure $\ref{fig:acc}$, performs not better than  G2DQPCA with $s=2,p=\infty$ indicated by $(2,{\rm inf})$  in most of cases.  The classification accuracy of  such 2DQPCA   is also greatly less than  its  weighted version, indicated by $w(2,2)$.  That indicates that the proposed  weighted G2DQPCA algorithm performs much better than the state-of-the-art 2DQPCA-based approaches in color face  recognition.

\end{example}

\begin{example}[Color face recognition with noisy training data]
In this example, we examine the recognition performance of G2DQPCA and WG2DQPCA with the training set  polluted by different kinds and levels of noise,  while the testing set are clean. On the three databases, two different noises are considered separately. In the first case, the noise consists of random black and white dots. We randomly added this noise to $20\%$ of the training set. The location of noise is arbitrary, and the size is at least $10\times 10$.  The second case is to add salt and pepper noise to the training set. Here we set the noise density as 0.02, 0.05 and 0.1, respectively.

\begin{figure*}[!t]
  \setlength{\abovecaptionskip}{0.1cm}
  \setlength{\belowcaptionskip}{0.1cm}
  \centering
  \includegraphics[width=0.99\textwidth,height=0.36\textwidth]{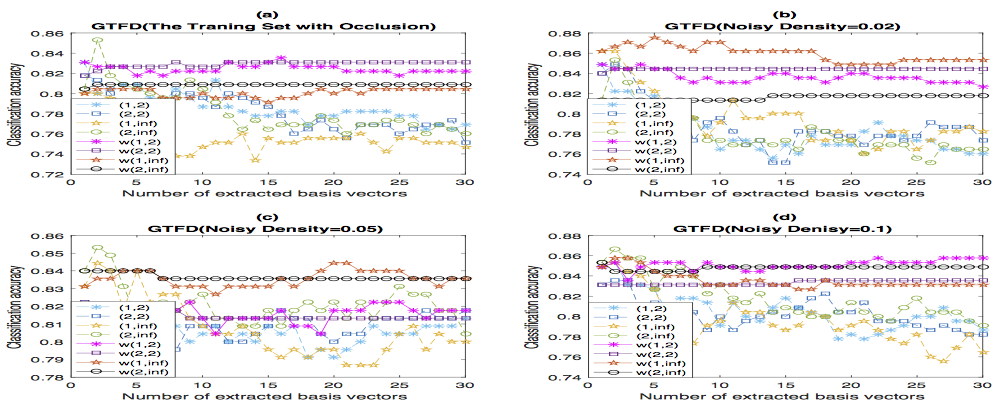}
\caption{Classification accuracies with polluted training set of G2DQPCA and WG2DQPCA on the GTFD database.
}
\label{acc:gtfd_noisytrain}
\end{figure*}

\begin{figure*}[!t]
  \setlength{\abovecaptionskip}{0.1cm}
  \setlength{\belowcaptionskip}{0.1cm}
  \centering
  \includegraphics[width=0.99\textwidth,height=0.36\textwidth]{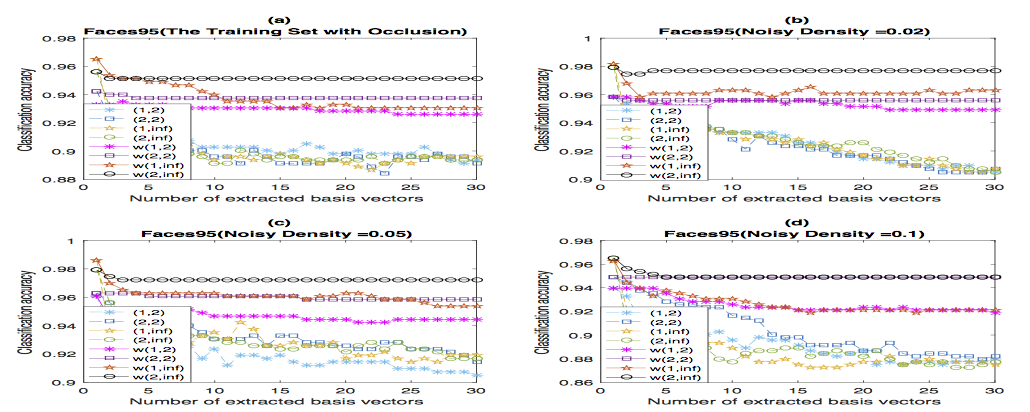}
\caption{Classification accuracies with polluted training set of G2DQPCA and WG2DQPCA on the Faces95 database.
}
\label{acc:faces95_noisytrain}
\end{figure*}

\begin{figure*}[!t]
  \setlength{\abovecaptionskip}{0.1cm}
  \setlength{\belowcaptionskip}{0.1cm}
  \centering
  \includegraphics[width=0.99\textwidth,height=0.36\textwidth]{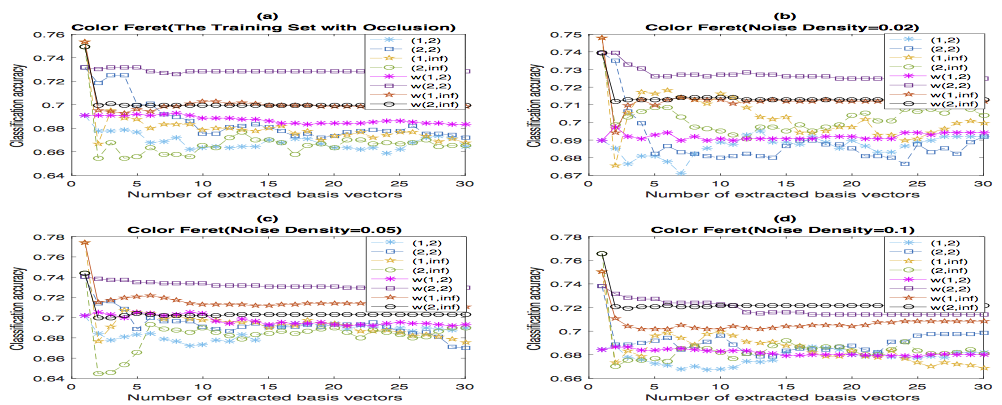}
\caption{Classification accuracies with polluted training set of G2DQPCA and WG2DQPCA on the Color FERET database.
}
\label{acc:feret_noisytrain}
\end{figure*}

The whole procedure of each case is repeated three times and the average result is reported. In Figures $\ref{acc:gtfd_noisytrain}$, $\ref{acc:faces95_noisytrain}$ and $\ref{acc:feret_noisytrain}$, we present the average classification accuracy of compared methods  when the training set is polluted by different kinds and levels of noises:

{\bf(a)}  The training set is polluted  by black and white dots noises.

{\bf(b)-(d)} The training set is polluted by  salt and pepper noises and  the noise densities are 0.02, 0.05 and 0.1, respectively.

\noindent
The numerical results indicate that the pollution on the training set  impacts the classification performance since the feature extraction is perturbed by noise. 

It is still true that the recognition rate of WG2DQPCA is more stable and higher than that of G2DQPCA.  Obviously, under the influence of different noises, the optimal parameter pair $w(s,p)$ is not fixed for all databases.
 In Figure $\ref{acc:gtfd_noisytrain}$,  it is shown that for GTFD database, WG2DQPCA with $s=2$ and $p=2$ has good performance in the case of black and white dots noise; WG2DQPCA with $s=1$ and $p=\infty$ works well when the salt and pepper noise density are 0.02 and 0.05; and WG2DQPCA with $s=1$ and $p=2$ performs better in the last case.
 For Faces95 database,  the numerical results in Figure $\ref{acc:faces95_noisytrain}$ indicate that WG2DQPCA with $s=2$ and $p=\infty$ achieves the highest classification accuracy in all cases.  As for Color FERET database, it is observed that the performance of WG2DQPCA with $s=2$ and $p=2$ is relatively stable from Figure $\ref{acc:feret_noisytrain}$.  It can be concluded that the weighing technique not only promotes the performance of G2DQPCA,  but also helps G2DQPCA  to resist the pollution of noises.

\end{example}

\begin{example}[Color image reconstruction]

\begin{figure*}[!t]
   \setlength{\abovecaptionskip}{0.1cm}
  \setlength{\belowcaptionskip}{0.1cm}
  \centering
  \includegraphics[width=0.99\textwidth,height=0.4\textwidth]{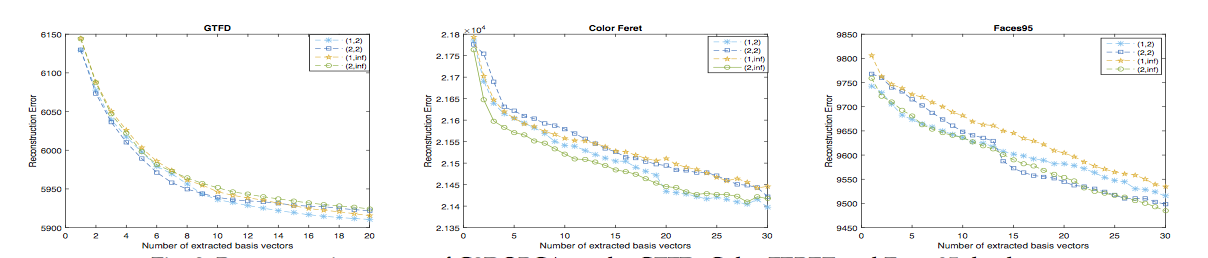}
\caption{Reconstruction errors of G2DQPCA on the GTFD, Color FERET and Faces95 databases.}
\label{fig:err}
\end{figure*}
\begin{figure*}[!t]
  \centering
    \setlength{\abovecaptionskip}{0.1cm}
  \setlength{\belowcaptionskip}{-0.2cm}
  \includegraphics[width=0.6\textwidth,height=0.2\textwidth]{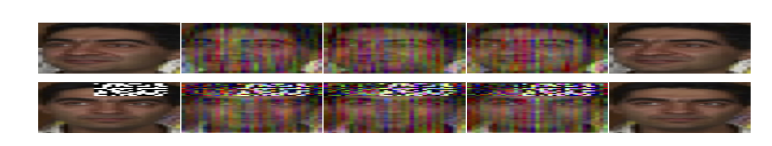}
\caption{The reconstructed images by G2DQPCA on two sample images from the polluted GTFD database.
}
\label{example:recoGTFD}
\end{figure*}
\begin{figure*}[!t]
  \centering
    \setlength{\abovecaptionskip}{0.1cm}
  \setlength{\belowcaptionskip}{-0.2cm}
  \includegraphics[width=0.6\textwidth,height=0.2\textwidth]{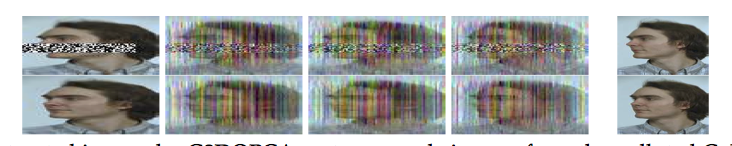}
\caption{The reconstructed images by G2DQPCA on two sample images from the polluted Color FERET database.
}
\label{example:recoferet}
\end{figure*}
\begin{figure*}[!t]
  \centering
    \setlength{\abovecaptionskip}{0.1cm}
  \setlength{\belowcaptionskip}{-0.2cm}
  \includegraphics[width=0.6\textwidth,height=0.2\textwidth]{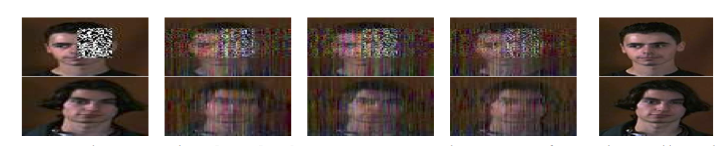}
\caption{The reconstructed images by G2DQPCA on two sample images from the polluted Faces95 database.
}
\label{example:faces95}
\end{figure*}

In this example, we evaluate the reconstruction performance of G2DQPCA algorithm. Here $20\%$ of the three databases have randomly selected to add black and white dots noise, respectively.

The following reconstruction error is used to measure the quality of methods:
\begin{equation}\label{reco-err}
  err=\frac{1}{\ell}\sum\limits_{i=1}^{\ell}\|\Fq_{i}^{\textit{clean}}(\Iq-\Wq\Wq^{*})\|_{F}
\end{equation}
where $\ell$ is the number of clean training data, $\Fq_{i}^{\textit{clean}} $ is the $i$-th clean training sample and $\Wq$ is the projection matrix trained on the whole polluted training database. As emphasized in Section $\ref{s:application}$, $\Wq$ is unweighted  because the process of reconstruction does not need to magnify the effect of features.

In the first graph of Figure $\ref{fig:err}$, we show the reconstruction errors of G2DQPCA in several special cases with the number of features from $1$ to $20$ on GTFD database. In the second and third graphs of Figure $\ref{fig:err}$, we present the reconstruction errors of G2DQPCA with the feature number from $1$ to $30$ on the other two databases.
From the numerical results, all parameter pairs are indicated to have a good performance in image reconstruction.
With the increase of features number, it is observed that applying $L_{1}$-norm on the objective function, i.e., $s=1,p=2$, is effective for the images reconstruction on GTFD database and Color FERET database. G2DQPCA works well on the Faces95 database when $s=2, p=2$ and $s=2, p=\infty$.

As an illustration,  several reconstructed images are listed in Figures $\ref{example:recoGTFD}$, $\ref{example:recoferet}$ and $\ref{example:faces95}$.  The first column are the images to be reconstructed. The following three columns are the reconstructed images by using the first $20$ projection vectors of G2DQPCA wherein the $(s, p)$ pairs are set to be $(1,2)$, $(2,2)$ and $(2,1)$ in order. The last column shows the original images for comparison.

\end{example}

\begin{example}[Comparison with deep learning methods]

\begin{figure*}[!t]
  \setlength{\abovecaptionskip}{0.1cm}
  \setlength{\belowcaptionskip}{0.1cm}
  \centering
  \includegraphics[width=0.99\textwidth,height=0.4\textwidth]{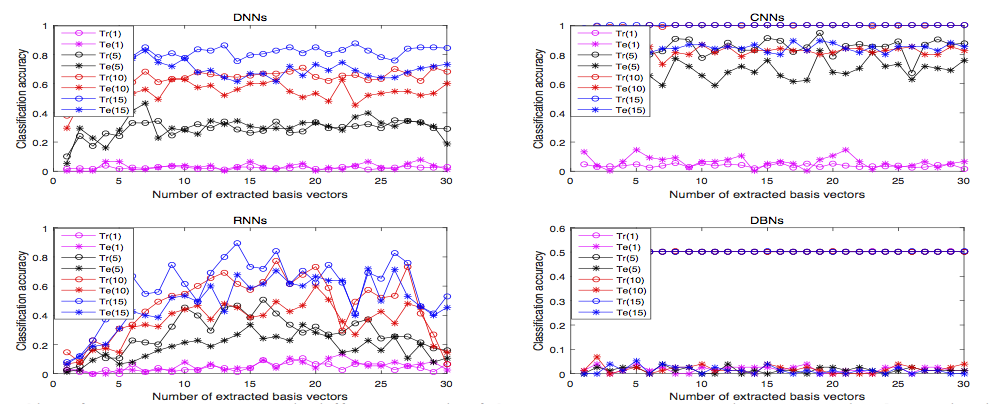}
\caption{Classification accuracies with different epoch of CNNs, DNNs, DBNs and RNNs on the GTFD database.
}
\label{fig:epoch}
\end{figure*}
\begin{figure*}[!t]
  \setlength{\abovecaptionskip}{0.1cm}
  \setlength{\belowcaptionskip}{0.1cm}
  \centering
  \includegraphics[width=0.99\textwidth,height=0.4\textwidth]{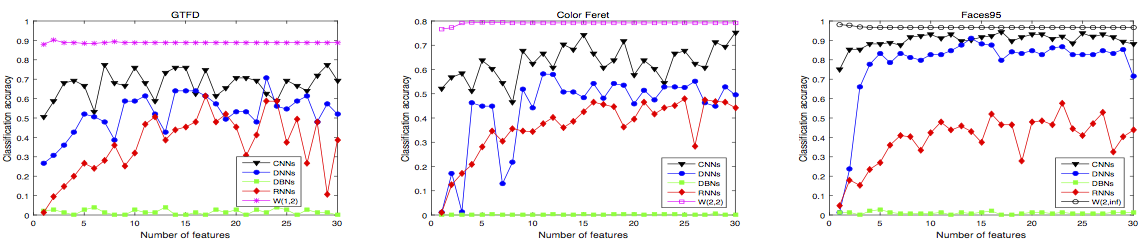}
\caption{Classification accuracies of WG2DQPCA,CNNs,DNNs,DBNs and RNNs on the GTFD, Color FERET and Faces95 databases.}
\label{fig:compare}
\end{figure*}

In this experiment, we compare WG2DQPCA with four  prominent deep learning primitives: Convolutional Neural Networks (CNNs), Deep Neural Networks (DNNs), Deep Belief Networks (DBNs) and Recurrent Neural Networks (RNNs). 
\begin{itemize}
\item DNNs are implemented by stacking layers of neural networks along the depth and width of smaller architectures. For GTFD and Faces95 databases, we increase the number of hidden neurons in first DNN from 8 to $588$, step $20$, and the number of neurons in second and third hidden layers are set as $512$ and $512$. For Color FERET database, we increase the number of hidden neurons in first DNN from $208$ to $788$ with step $20$, and the number of neurons in second and third hidden layers are set as $512$ and $512$.
\item CNNs  are feed-forward, back-propagate neural networks with a special architecture inspired from the visual system, consisting
of alternating layers of convolution layers and sub-sampling layers. In our experiments, we set a convolution layer and a sub-sampling layer. The convolution layer has k feature maps, where we increase $k$ from $8$ to $588$ with step $20$, connected to the single input layer through $k$ $5\times 5$ kernels.
\item DBNs consist of a number of layers of Restricted Boltzmann Machines (RBMs) which are trained in a greedy layer wise fashion. We increase the number of hidden neurons in first RBM from $8$ to $588$ with step $20$ and set a fixed number of hidden neurons of $256$ in second RBM.
\item RNNs  is a kind of recurrent neural networks which takes sequence data as input, recurses in the evolution direction of sequence, and all nodes are connected in chain. We increase the number of hidden neurons from $8$ to $588$ with step $20$ .
\end{itemize}
The comparison is implemented  on color face recognition with using the GTFD, Color FERET, and Faces95 databases. We randomly select 90 percent of data set as the training set and the remaining as the testing set.   

From Example \ref{ex:1}, it has been observed that WG2DQPCA obtains  optimal recognition rates on  the GTFD, Color FERET, and Faces95 databases  when  $w$($s=1$, $p=2$), $w$($s=2$, $p=2$) and $w$($s=2$, $p=\infty$), respectively.  So we fix these optimal values of $s$ and $p$ for   WG2DQPCA during the comparison.

In the test process, we found that if the number of iterations (epoch) is small, the recognition rates of neural networks are low. With increasing iteration times, their recognition rates will increase, but they are easy to overfit. Overfitting means that the model has a small error on the training set but a large error on the test set. Take the GTFD database as an example. The
numerical results are shown in Figure $\ref{fig:epoch}$. 
 The lines $'-o-'$ and $'-\ast-'$ represent the classification accuracies of training and testing samples, respectively. The fuchsin, black, red and blue lines represent the classification accuracies with epoch=$1$, $5$, $10$, $15$. 
   From Figure $\ref{fig:epoch}$, we can see that  the neural networks have high accuracy but have only slight overfitting phenomenon when the epochs of  DNNs, CNNs and RNNs are respectively $10$, $5$ and $10$, respectively. After that, the overfitting  happens, i.e., the  $'-o-'$ line becomes higher than the $'-\ast-'$ line.  The recognition rates of DBNs are always low. So in this experiment, set epoch=$10$ for DNNs, DBNs and RNNs, and epoch=$5$ for CNNs.  It should be noticed that the recognition rates of CNNs, DNNs, DBNs and RNNs can not achieve at high levels with small samples.

 The numerical results are shown in Figure  $\ref{fig:compare}$. 
  The average recognition rates of WG2DQPCA, CNNs, DNNs, DBNs and RNNs are  0.89, 0.67, 0.54, 0.01 and 0.37, respectively; and the maximum recognition rates are  0.9, 0.77, 0.71, 0.03 and 0.59. It indicates that WG2DQPCA performs better than CNNs, DNNs, DBNs and RNNs on color face recognition.

\end{example}

\bigskip
The above performance of G2DQPCA and WG2DQPCA can be understood by the following analysis.
G2DQPCA is in fact to extract a number of principal components (features) of all color images in the training set,   which can be used to reconstruct the optimal  low-rank approximations of original color images. The projections under the exacted features has the largest variance measured by the $L_p$ norm in the objective function. The proposed weighting technique enhances the role of principal components and thus WG2DQPCA performs better than G2DQPCA in color image recognition.  Generally,  the noise influence often exists in the high frequency components, but not in the principal (low frequency) components. That is why the reconstruction has the efficiency of noise reduction.

\section{Conclusion}\label{s:conclusion}
\noindent
In order to extract more types of geometric features of color images, this paper proposes a generalized two-dimensional quaternion principal component analysis method and a color image recognition model based on this method. The quaternion optimization problem in the model is equivalently transformed into a real optimization problem and a fast algorithm is proposed as well as the convergence analysis theory. The proposed principle component weighting strategy helps overcome the odd phenomenon that the recognition rate of the color image recognition model based on the traditional two-dimensional quaternion principal component analysis method decreases with the increase of principal components. Numerical experiments indicate that the proposed method has a higher face recognition rate than the state-of-the-art methods of the same type and four well-known deep learning methods in color image recognition. Even so, our method still has several aspects to be improved. For instance, there is no theory to determine the optimal values of model parameters. A possible solution is to  use the validation set for data-driven parameter selection.


\newpage 
\section{Appendix}\label{s:appendix}
\noindent
To be self-contained, we present some important results in this appendix.

\subsection{Properties of convex and concave functions}
The following theorem explains the property of convex functions [47].
\begin{theorem}\label{cf_pf}
Let $S$ be a non-empty open convex set in $\mathbb{R}^{n}$, and let $f$ be a function with a first-order continuous derivative function, i.e.,$f\in C^1(S)$. Then $f$ is a convex function if and only if for all $w$, $v\in S$,
\begin{equation}\label{first_order}
f(w)\geq f(v)+\nabla f(v)^T(w-v).
\end{equation}
\end{theorem}
\begin{proof}
Let $f$ be a convex function. In other words, for all $w$, $v\in S$, $\lambda\in (0,1)$,
\begin{equation}\label{small}
\begin{aligned}
f(\lambda w+(1-\lambda)v)&\leq\lambda f(w)+(1-\lambda)f(v).
\end{aligned}
\end{equation}
Meanwhile,
\begin{equation*}
\begin{aligned}
f(\lambda w+(1-\lambda)v)&=f(v+\lambda(w-v))\\
                                               &=f(v)+\lambda \nabla f(v)^T(w-v)\\
                                               &+O(\lambda\|w-v\|),
\end{aligned}
\end{equation*}
where $O(\lambda\|w-v\|)\longrightarrow0$ when $\|w-v\|\longrightarrow0.$

Dividing by $\lambda$ on both sides, we get $$f(w)\geq f(v)+\nabla f(v)^T(w-v).$$

Conversely, since $S$ is convex, then for all $w$, $v\in S$, $\lambda\in (0,1)$, let $z=\lambda w+(1-\lambda)v\in S$.  By the condition \eqref{first_order}, we have that
\begin{subequations}
\begin{align}
f(w)&\geq f(z)+\nabla f(z)^T(w-z),\\
f(v)&\geq f(z)+\nabla f(z)^T(v-z).
\end{align}
\end{subequations}
Thus, combining the above equations, we get
\begin{equation*}
\begin{aligned}
\lambda f(w)&+(1-\lambda)f(v)\geq\lambda f(z)+\lambda\nabla f(z)^T(w-z)\\
&+(1-\lambda) f(z)+(1-\lambda)\nabla f(z)^T(v-z) \\
&=f(z)+\nabla f(z)^T(\lambda w+(1-\lambda)v-z)\\
&=f(\lambda w+(1-\lambda)v).
\end{aligned}
\end{equation*}

Therefore, $f$ is a convex function.
\end{proof}

Next, we consider concave functions.
\begin{theorem}\label{concave_pf}
Let $S$ be a non-empty open convex set in $\mathbb{R}^{n}$, $f\in C^1(S)$. Then $f$ is a concave function if and only if for all $w$, $v\in S$,
\begin{equation}\label{concave-function}
f(w)\leq f(v)+\nabla f(v)^T(w-v).
\end{equation}
\end{theorem}

\begin{proof}
The proof of Theorem \ref{concave_pf}  is similar to that of Theorem \ref{cf_pf}. The only difference is that we need to replace the "$\geq$" with "$\leq$" in \eqref{small} as well as in other formulas.
\end{proof}

\subsection{Proof of  equation $(3.5)$}
By the definition, we can get 
\begin{equation*}
\begin{aligned}
  |\vq^{(\gamma)}|_Q^{p-1}&=(|\vq|^{p-1}+|\vq|^{p-1}\iq+|\vq|^{p-1}\jq+|\vq|^{p-1}\kq)^{(\gamma)}\\
          &= \left[
           \setlength{\arraycolsep}{0.15em}
            \begin{array}{llll}
              (|\vq|^{p-1})^T&(|\vq|^{p-1})^T&(|\vq|^{p-1})^T&(|\vq|^{p-1})^T
            \end{array}
            \setlength{\arraycolsep}{2pt}
          \right]^T,
\end{aligned}
\end{equation*}
and
\begin{equation*}
\begin{aligned}
  {\rm signQ}(\vq^{(\gamma)})&=({\rm sign}(\vq))^{(\gamma)}\\
          &= \left[
           \setlength{\arraycolsep}{0.2em}
            \begin{array}{llll}
              [\frac{v_i^{(0)}}{|\vq_i|}]_n^T&[\frac{v_i^{(1)}}{|\vq_i|}]_n^T&[\frac{v_i^{(1)}}{|\vq_i|}]_n^T&[\frac{v_i^{(3)}}{|\vq_i|}]_n^T
            \end{array}
            \setlength{\arraycolsep}{5pt}
          \right]^T.
\end{aligned}
\end{equation*}
Thus, we have 
\begin{equation}\label{abs_sign}
\begin{aligned}
 &|\vq^{(\gamma)}|_Q^{p-1}\circ {\rm signQ}(\vq^{(\gamma)})\\
=&\left[
           \setlength{\arraycolsep}{0.15em}
            \begin{array}{llll}
              [|\vq_i|^{p-2}v_i^{(0)}]_n^T & [|\vq_i|^{p-2}v_i^{(1)}]_n^T  &[|\vq_i|^{p-2}v_i^{(2)}]_n^T  &[|\vq_i|^{p-2}v_i^{(3)}]_n^T
            \end{array}
            \setlength{\arraycolsep}{3.5pt}
          \right]^T.
\end{aligned}
\end{equation}
Above all, it  naturally  derives that
\begin{equation*}
\begin{aligned}
 &\quad(|\vq^{(\gamma)}|_Q^{p-1}\circ {\rm signQ}(\vq^{(\gamma)}))^T\vq^{(\gamma)}\\
&=\sum_{i=1}^{n}|\vq_i|^{p-2}[(v_i^{(0)})^2+(v_i^{(1)})^2+(v_i^{(2)})^2+(v_i^{(3)})^2]\\
&=\sum_{i=1}^{n}|\vq_i|^{p-2}|\vq_i|^2=\sum_{i=1}^{n}|\vq_i|^{p}=\|\vq^{(\gamma)}\|_{2,p}^p.
\end{aligned}
\end{equation*}
\subsection{Proof of equation $(3.22)$}
From equation $(3.21)$, it's not difficult to understand that 
\begin{equation*}
(\uq^{(\gamma)})^{(k)}=|(\vq^{(\gamma)})^{(k)}|_Q^{q-1}\circ {\rm signQ}((\vq^{(\gamma)})^{(k)}).
\end{equation*}
Then, we give an explanation
\begin{equation*}
\|(\vq^{(\gamma)})^{(k)}\|_{2,q}^{q-1}=\|(\uq^{(\gamma)})^{(k)}\|_{2,p}.
\end{equation*}
By the equation \eqref{abs_sign}, 
\begin{equation*}
\begin{aligned}
\|\uq^{(\gamma)}\|_{2,p}&=\||\vq^{(\gamma)}|_Q^{q-1}\circ {\rm signQ}(\vq^{(\gamma)})\|_{2,p}\\
&=\||\vq^{(\gamma)}|^{q-1}\|_{2,p}.
\end{aligned}
\end{equation*}
Recalling the definition of $L_{2,p'}$-norm, we obtain
\begin{equation*}
\|(\uq^{(\gamma)})^{(k)}\|_{2,p}=\left[\sum_{i=1}^n(|(v_i^{(\gamma)})^{(k)}|^{q-1})^{p}\right]^{\frac{1}{p}}.
\end{equation*}
Let $\frac{1}{p}+\frac{1}{q}=1$, then
\begin{equation*}
\begin{aligned}
\|(\uq^{(\gamma)})^{(k)}\|_{2,p}&=\left[\sum_{i=1}^n|(v_i^{(\gamma)})^{(k)}|^q\right]^{\frac{q-1}{q}}\\
&=\left[\left(\sum_{i=1}^n|(v_i^{(\gamma)})^{(k)}|^q\right)^{\frac{1}{q}}\right]^{q-1}\\
&=\|(\vq^{(\gamma)})^{(k)}\|_{2,q}^{q-1}.
\end{aligned}
\end{equation*}

\subsection{Proof of equation $(3.29)$}
Now we consider the partial derivative with respect to $\wq^{(\gamma)}$:
\begin{equation*}
\left[\frac{\partial h(\wq^{(\gamma)})}{\partial \wq^{(\gamma)}}\right]_l=\frac{\partial h(\wq^{(\gamma)})}{\partial w_i^{(j)}}.
\end{equation*}
Considering the partial derivative
\begin{equation*}
\begin{aligned}
\left[\frac{\partial |\wq|}{\partial w_i^{(j)}}\right]_l&=\frac{\partial}{\partial w_i^{(j)}}|\wq_i|=\frac{w_i^{(j)}}{|\wq_i|}.
\end{aligned}
\end{equation*}
Then
\begin{equation*}
\left[\frac{\partial h(\wq^{(\gamma)})}{\partial \wq^{(\gamma)}}\right]_l=s(\vq^{(k)})_i^{(j)}-\lambda p\left[|\wq_i^{(k)}|^{p-1}\frac{w_i^{(j)}}{|\wq_i|}\right].
\end{equation*}

\end{document}